%% file: main.tex
\documentclass[lettersize,journal, onecolumn]{IEEEtran}
\usepackage{amsmath,amsfonts, amssymb, amsthm}
\usepackage{mathtools}
\usepackage{algorithmic}
\usepackage{array}
\usepackage[caption=false,font=normalsize,labelfont=sf,textfont=sf]{subfig}
\usepackage{textcomp}
\usepackage{stfloats}
\usepackage{url}
\usepackage{verbatim}
\usepackage{graphicx}
\def\BibTeX{{\rm B\kern-.05em{\sc i\kern-.025em b}\kern-.08em
    T\kern-.1667em\lower.7ex\hbox{E}\kern-.125emX}}
\usepackage{balance}

\usepackage[table]{xcolor}
\definecolor{my_gray}{HTML}{DBDBDB}
\renewcommand{\arraystretch}{1.7}
\newtheorem{assumption}{Assumption}
\newtheorem{theorem}{Theorem}
\newtheorem{example}{Example}
\newtheorem{remark}{Remark}

\newtheorem{corollary}{Corollary}
\newtheorem{lemma}{Lemma}

\counterwithin{assumption}{section}

\DeclareMathOperator*{\argmin}{arg\,min}
\newcommand{\x}{\mathbf{x}}
\newcommand{\var}{\mathrm{var}}
\newcommand{\tr}{\text{tr}}
\newcommand{\W}{\mathbf{W}}
\newcommand{\y}{\mathbf{y}}
\newcommand{\J}{\mathbf{J}}
\newcommand{\V}{\mathbf{V}}
\newcommand{\boldtheta}{\boldsymbol{\theta}}
\newcommand{\boldSigma}{\boldsymbol{\Sigma}}

\newcommand{\numIneq}[1]{\stackrel{\footnotesize \mathclap{\mbox{$#1$}}}{\leq}}
\newcommand{\numStrictIneq}[1]{\stackrel{\footnotesize \mathclap{\mbox{$#1$}}}{<}}
\newcommand{\numEq}[1]{\stackrel{\footnotesize \mathclap{\mbox{$#1$}}}{=}}

\begin{document}
\title{Bayesian Cram\'{e}r-Rao Bound Estimation with Score-Based Models}
\author{Evan Scope Crafts, Xianyang Zhang, and Bo Zhao
\thanks{This work was supported in part by the National Institutes of Health under grants NIH R00EB027181. and NIH R01GM144351 and by the National Science Foundation under grant NSF DMS2113359. An earlier version of this paper was presented in part at the 2023 IEEE International Conference on Acoustics, Speech, and Signal Processing
(ICASSP-2023) [DOI: 10.1109/ICASSP49357.2023.10095110]. ({\it Corresponding Author: } Bo Zhao.)

Evan Scope Crafts is with the Oden Institute for Computational Engineering and Sciences, The University of Texas at Austin,  Austin, TX 78712, USA (e-mail: escopec@utexas.edu). 

Xianyang Zhang is with the Department of Statistics, Texas A\&M University, College Station, TX 77843, USA (e-mail: zhangxiany@stat.tamu.edu). 

Bo Zhao is with the Oden Institute for Computational Engineering and Sciences and the Department of Biomedical Engineering, The University of Texas at Austin,  Austin, TX 78712, USA (e-mail: bozhao@utexas.edu). 

}}

\maketitle

\begin{abstract}
The Bayesian Cram{\'e}r-Rao bound (CRB) provides a lower bound on the mean square error of any Bayesian estimator under mild regularity conditions. It can be used to benchmark the performance of statistical estimators, and provides a principled metric for system design and optimization. However, the Bayesian CRB depends on the underlying prior distribution, which is often unknown for many problems of interest. This work introduces a new data-driven estimator for the Bayesian CRB using score matching, i.e., a statistical estimation technique that models the gradient of a probability distribution from a given set of training data. The performance of the proposed estimator is analyzed in both the classical parametric modeling regime and the neural network modeling regime. In both settings, we develop novel non-asymptotic bounds on the score matching error and our Bayesian CRB estimator based on the results from empirical process theory, including classical bounds and recently introduced techniques for characterizing neural networks. We illustrate the performance of the proposed estimator with two application examples: a signal denoising problem and a dynamic phase offset estimation problem in communication systems.
\end{abstract}


\input{Sections/Intro}

\input{Sections/Background}
\input{Sections/ProposedEstimator}
\input{Sections/classicalconverge}

\input{Sections/convergeNN}

\input{Sections/NumExperiments}

\input{Sections/conclusion}

\bibliographystyle{IEEEtran.bst}

\bibliography{ref.bib}

\begin{appendices}

\input{Appendices/classical}
\input{Appendices/nn}

\end{appendices}

\end{document}

%% file: Sections/Intro.tex
\section{Introduction}
\label{sec:intro}

The Cram{\'e}r-Rao bound (CRB) \cite{rao_1945_information, cramer1946} provides a lower bound on the mean square error (MSE) of any unbiased estimator in classical statistical inference. It is widely used to benchmark the performance of statistical estimation (e.g., \cite{Stoica1989, zhao_2014_model, diskin2023learning}) and to guide system design and optimization (e.g., \cite{li_2007_range, zhao_2019_optimal, crafts2022}). As a generalization of the CRB, the Bayesian Cram{\'e}r-Rao bound \cite{bell_2013_detection, trees_2007_bayesian} (also known as the posterior or Van Trees CRB) provides a lower bound on the mean square error (MSE) of Bayesian estimators. The Bayesian CRB exists under mild technical conditions and has found a number of applications, including image registration \cite{Robinson2004, Aguerrebere2016}, dynamic system analysis \cite{Tichavsky1998}, and communication array design \cite{Oktel2005}.

Analytic computation of the Bayesian CRB requires knowledge of the score (i.e.,  the derivative of the log-density) of the prior distribution and likelihood function. While for many applications of interest (e.g., in medical imaging \cite{prince2015medical} and communications \cite{Madhow_2014}), the likelihood function can be determined from physical principles, the prior distribution is often more complex and difficult to model. For example, in various inverse problems in imaging (e.g., image deconvolution or tomographic imaging), it has long been an open problem to model the probability distribution for certain image classes of interest. While generic or hand-crafted priors (e.g., \cite{Bouman1993, Babacan2010}) can be adopted, this could lead to substantial error in the Bayesian CRB calculation due to oversimplification and misspecification of the prior distribution, and, as a result, the estimated Bayesian CRB may not be able to accurately characterize the performance of statistical estimators that incorporate the true prior distribution. Motivated by these challenges and the ever-increasing availability of large data sets \cite{deng2009}, this paper introduces a novel Bayesian CRB estimator with provable guarantees in the setting in which the likelihood function of the data model is known but only samples from the prior distributions are available.

\subsection{Estimation of Cram{\'e}r-Rao Type Bounds}

Existing estimators of Cram\'er-Rao type bounds generally focus on the classical CRB and its inverse, the Fisher information. These estimators can be broadly split into two categories: plug-in approaches that first estimate the likelihood function and form the CRB using this estimate, and direct approaches that estimate the CRB without first estimating the likelihood \cite{duy_fisher_2023}.

Plug-in approaches can use parametric or non-parametric estimators of the likelihood. An example of a non-parametric approach can be found in \cite{har_nonparametric_2016}, which uses density estimation using field theory (DEFT) \cite{kinney2014estimation} to estimate the underlying distribution. The authors demonstrate that this approach yields good performance on a univariate Gaussian distribution. However, the work does not provide any theoretical guarantees, and the approach does not scale well to high dimensions due to the curse of dimensionality inherent to DEFT. 

A parametric plug-in approach was recently introduced by Habi et al \cite{habi_generative_2022, habi_learning_2023}. The main idea behind their approach is the use of a conditional normalizing flow to model the likelihood function. The authors demonstrate experimentally that this approach is accurate for Gaussian and non-Gaussian measurement models and can provide useful information for image denoising and image edge detection tasks. They also provide non-asymptotic bounds on the estimator error. However, their bounds rely on strong assumptions (e.g., the score of the likelihood function is required to be bounded, which does not hold for many widely used likelihood functions such as Gaussians), and the bounds depend on the total variation distance between the learned generative model and the true measurement distribution, which is unknown.     

Direct approaches for CRB estimation exploit the relationship between $f$-divergences and the Fisher information \cite{berisha_empirical_2015, duy_fisher_2023}. Specifically, these approaches require first estimating the $f$-divergence between the original likelihood function and a perturbed version of the likelihood for different choices of perturbations. Here the $f$-divergence can be estimated using the Friedman-Rafskay (FR) multivariate test statistic \cite{berisha_empirical_2015} or a neural network based mutual information estimator \cite{duy_fisher_2023}. A semi-definite program is then solved to estimate the Fisher information from the $f$-divergence estimates. In \cite{duy_fisher_2023}, it was shown that these approaches can also be used for Bayesian CRB estimation.

While the above direct approaches are attractive because they avoid the estimation of the infinite-dimensional density function required by plug-in methods, they come with several practical difficulties. Specifically, both approaches require the estimation of the $f$-divergence for at least $D (D + 1) / 2$ perturbations where $D$ is the dimension of the unknown parameters, which is computationally expensive in high dimensions and, in the case of \cite{duy_fisher_2023}, requires the optimization of a separate neural network for each choice of perturbation. The approaches also do not have provable convergence guarantees. 

\subsection{Our Contributions}

Inspired by recent advances in generative modeling, this work introduces a novel Bayesian CRB estimator with non-asymptotic convergence guarantees. Specifically, the proposed approach employs score matching \cite{hyvarinen_2005_estimation} to directly estimate the score of the unknown prior distribution for the estimation of the Bayesian information matrix and the Bayesian CRB. Score matching is a statistical score estimation technique that minimizes the distance between the scores of the data and model distribution. Compared with the existing plug-in approaches that use density estimation techniques such as maximum likelihood estimation, score matching has the advantage of being independent of the model's normalizing constant, which is often intractable. Note that score matching is also a key component of diffusion models, which have achieved state-of-the-art results in generative modeling \cite{dhariwal_2021_diffusion}. 



We characterize the performance of the proposed Bayesian CRB estimator in two modeling regimes. The first regime corresponds to a classical parametric modeling setting where the number of model parameters is less than the number of samples from the prior distribution. The second regime considers the case where the score model is a neural network. In both regimes, the key challenge lie in characterizing the score matching error. Here we develop novel non-asymptotic error bounds for score matching in both regimes. Our proofs are based on several results from empirical process theory, including the rate theorem \cite{wellner_1996_weak},  Talagrand’s inequality for empirical processes \cite{bartlett_2005_local}, and neural network covering bounds \cite{bartlett2017spectrally}. To the best of our knowledge, this work is the first to provide non-asymptotic bounds on score matching with general neural network models. 



We illustrate the performance of our proposed Bayesian CRB estimator with two application examples: a denoising problem with a Gaussian mixture prior and a dynamic phase-offset estimation problem in communication with a prior given by a Wiener phase-offset evolution.

\subsection{Organization}

The remainder of the paper is organized as follows. Section \ref{sec:back_prelim} introduces the technical backgrounds on the Bayesian CRB and score matching as well as mathematical notations. Section \ref{sec:proposed_approach} presents the problem formulation and the proposed Bayesian CRB estimator. Sections \ref{sec:classicalconverge} and \ref{sec:convergeNN} describe the covergence results for the classical setting and the neural network setting, respectively. Section \ref{sec:num_experiments} contains numerical results from the application examples, followed by the concluding remarks in Section \ref{sec:conclusion}.

%% file: Sections/Background.tex
\section{Background and Preliminaries}
\label{sec:back_prelim}

\subsection{Notation}

We use bold letters to denote vectors (e.g., $\x$) and capital bold letters to denote matrices and higher-order tensors (e.g., $\mathbf{X}$). The gradient of a scalar-valued function $f(\x): \mathbb{R}^N \to \mathbb{R}$ is denoted $\nabla_{\x} f(\x)$ and is written as an $N \times 1$ vector. The Jacobian of a vector-valued function $F(\x): \mathbb{R}^N \to \mathbb{R}^M$ is written as an $M \times N$ matrix, denoted $\nabla_{\x} F(\x)$, with $[\nabla_{\x} F(\x)]_{i, j} = \partial F(\x)_i / \partial \x_j$. We use $\x_1^N$ as shorthand for the data set $\{ \x_1, \cdots, \x_N \}$.

We use $\mathbf{X}^T$, $\text{tr}(\mathbf{X})$, $\mathbf{X}^{-1}$, and $\mathbf{X}^{\dagger}$ to respectively denote the transpose, trace, inverse, and Moore-Penrose pseudoinverse of a given  matrix $\mathbf{X}$. For square matrices $\mathbf{A}$ and $\mathbf{B}$, the generalized inequality $ \mathbf{A} \succcurlyeq \mathbf{B}$ means that $\mathbf{A} - \mathbf{B}$ is a positive semidefinite matrix.  We use $\| \cdot \|_\sigma$ to denote the spectral norm of a given matrix, while $\| \cdot \|_2$ is used to denote the component-wise two-norm of a given tensor; thus, for a matrix, it corresponds to the Frobenius norm. The expression $\| \cdot \|_{p, q}$ denotes the $(p, q)$ matrix norm, defined by $\| \mathbf{X} \|_{p, q} = \| [\|\mathbf{X}_{:, 1}\|_p, \cdots, \| \mathbf{X}_{:, M} \|_p ]^T \|_q$ for $\mathbf{X} \in \mathbb{R}^{N \times M}$. The bilinear expression $\langle \cdot, \cdot \rangle$ refers to the standard Euclidean inner product.

For random vectors $\x$ and $\y$, we use $p(\x)$ to denote the density function of $\x$, $p(\y | \x)$ to denote the conditional density of $\y$ given $\x$, and $p(\x, \y)$ to denote the joint density function. The term $\mathbb{E}_{\x}$ denotes the expectation with respect to $\x$, $\mathbb{E}_{\x, \y}$ denotes the expectation with respect to the joint distribution of $\x$ and $\y$, and $\mathbb{E}_{\y | \x}$ denotes the conditional expectation of $\y$ given $\x$. The $\epsilon$-covering number of a set $A$ with respect to a given norm $d$ is denoted $N(A, d, \epsilon)$. We use $\log (\cdot)$ to refer to the natural logarithm. 
\subsection{The Bayesian CRB}

Let $\x \in \mathcal{X} \subset \mathbb{R}^D$ be a random parameter vector of interest, let $\y \in \mathcal{Y} \subset \mathbb{R}^K$ be observations from a model with parameters $\x$, and let $\hat{\x}(\y)$ be an estimator of $\x$. Assume the following regularity conditions hold.
\begin{assumption}[Support]\label{assumption:support}
The set $\mathcal{X}$ is either $\mathbb{R}^D$ or an open bounded subset of $\mathbb{R}^D$ with piecewise smooth boundary.
\end{assumption}
\begin{assumption}[Existence of Derivatives] \label{assumption_exist_finite}
The derivatives $[\nabla_{\x} p(\x, \y)]_i$, $i = 1, \dots, D$, exist and are absolutely integrable. 
\end{assumption}
\begin{assumption}[Finite Estimator Bias]
The bias of the estimator, i.e., $B(\x) \triangleq \int \left (\hat{\x}(\y) - \x\right)p(\y \mid \mathbf{\x}) \; d \y $, exists and is finite for all $\x$.
\end{assumption}
\begin{assumption}[Differentiation Under Integral Sign] The probability function $p(\x, \y)$ and estimator $\hat{\x}(\y)$ satisfy
$$
 \nabla_\x \int p(\x, \y) \left [\hat{\x}(\y) - \x \right]^T d\y  = \int \nabla_\x \left ( p(\x, \y) [\hat{\x}(\y) - \x]^T  \right )d\y 
$$
for all $\x$.

\end{assumption}
\begin{assumption}[Error Boundary Conditions]
\label{assumption:error_boundary}
Let $\partial \mathcal{X}$ denote the boundary of $\mathcal{X}$. For any sequence $\{\x_i \}_{i=0}^\infty$, $\x_i \in \mathcal{X}$, such that $\x_i \to \x \in \partial \mathcal{X}$, or any sequence such that $|| \x_i ||_2 \to \infty$, we have that $B(\x_i) p(\x_i) \to 0$.
\end{assumption}
Note that the above assumptions hold for a a wide class of ground-truth prior distributions, including Gaussian, Gaussian mixture, Laplacian, Student's t, and beta (with parameters $\alpha, \beta > 1$) distributions. Further, any differentiable probability distribution with bounded support that does not satisfy Assumption \ref{assumption:error_boundary}  can be made to satisfy the condition by convolving the distribution with a Gaussian of arbitrarily small variance.\footnote{To see this, note that convolving any distribution with a Gaussian yields a distribution with support equal to $\mathbb{R}^D$, and the convolved distribution must satisfy $p(\x_i) \to 0$ as $|| \x_i ||_2 \to \infty$ to be integrable.}

 Under the above assumptions, the Bayesian CRB can be defined via information inequality \cite{bell_2013_detection}:
\begin{equation*}
    \mathbb{E}_{\x, \y} \left [\left (\hat{\x}(\y) - \x \right )(\hat{\x} \left (\y) - \x \right )^T \right ] \succcurlyeq \V_B \triangleq \J_B^{-1}.
    \label{eq:BCRB}
\end{equation*}
Here $\V_B \in \mathbb{R}^{D \times D}$ denotes the Bayesian CRB and $\J_B  \in \mathbb{R}^{D \times D}$ is the Bayesian information, which can be written as 
\begin{equation*}
    \J_B \triangleq \mathbb{E}_{\x, \y} \left [  \nabla_\x \log p(\x, \y)  \nabla_\x \log p(\x, \y)^T \right ].
\end{equation*}
The Bayesian information can be decomposed into a prior-informed term $\J_P$ and a data-informed term $\J_D$, i.e., \cite{bell_2013_detection}:
\begin{equation}
    \J_B = \J_P + \J_D.
    \label{eq:BCRB_decomposition}
\end{equation}
Here $\J_P$ is defined as
\begin{equation*}
\J_P \triangleq \mathbb{E}_{\x} \left [ \nabla_\x \log p(\x) \nabla_\x \log p(\x)^T \right ],
\label{Jp}
\end{equation*}
while $\J_D$ is the average Fisher information associated with the observations, i.e.,
\begin{equation*}
\J_D \triangleq \mathbb{E}_{\x, \y} \left [ \nabla_\x \log p(\y |  \x) \nabla_\x \log p(\y | \x)^T \right ] = \mathbb{E}_{\x} \left [ \J_F(\x) \right ],
\label{Jd}
\end{equation*}
where $$\J_F(\x) \triangleq \mathbb{E}_{\y | \x} \left [\nabla_\x \log p(\y |  \x) \nabla_\x \log p(\y | \x)^T \right ]$$
is the Fisher information. 

\subsection{Score Matching}

Score matching \cite{hyvarinen_2005_estimation} is a statistical method for estimating the (Stein) score of an unknown data distribution $p(\x)$, i.e., $\nabla_\x \log p(\x)$, from a set of i.i.d. samples $\{\x_1, \x_2, \cdots, \x_N \} \subset \mathbb{R}^D$. Originally introduced by Hv\"{a}rinen and Dayan, the technique is based on minimizing the Fisher divergence between the data scores and the scores of a vector-valued model $s(\x; \boldtheta): \mathbb{R}^D \to \mathbb{R}^D$ parameterized by $\boldtheta \in \boldsymbol{\Theta}$:
\begin{equation}
    L(\boldtheta) \triangleq \frac{1}{2} \mathbb{E}_{\x}\left[ \| s(\x; \boldtheta) - \nabla_\x \log p(\x) \|_2^2 \right ].
    \label{eq:fisher_loss}
\end{equation}
Minimizing \eqref{eq:fisher_loss} directly is intractable since $\nabla_\x \log p(\x)$ is unknown. However, under the following regularity conditions, Hv\"{a}rinen proved that
$L(\boldtheta) = J(\boldtheta) + C$, where $C$ is a constant independent of $\boldtheta$ and\begin{equation}
J(\boldtheta) \triangleq \mathbb{E}_{\x} \left [ \tr\left(\nabla_{\x} s(\x; \boldtheta) \right ) + \frac{1}{2} \| s(\x; \boldtheta)\|_2^2 \right ].
\label{eq:sm_loss}
\end{equation}
\begin{assumption}[Regularity of Score Functions]
\label{assumption:regularity}
The score function estimate $s(\x; \boldtheta)$ and the true score function $\nabla_\x \log p(\x)$ are both differentiable with respect to $\x$. They additionally satisfy $\mathbb{E}_{\x} \left [ \|s(\x, \boldtheta)\|_2^2  \right] < \infty$ for all $\boldtheta \in \boldsymbol{\Theta}$ and $\mathbb{E}_{\x} \left [\nabla_\x \| \log p(\x) \|_2^2 \right ] < \infty$.
\end{assumption}
\begin{assumption}[Score Matching Boundary Conditions]
\label{assumption:boundary}
For any $\boldtheta \in \boldsymbol{\Theta}$ and any sequence $\{ \x_i \}_{i=0}^\infty$, $\x_i \in \mathcal{X}$, such that $\x_i \to \x \in \partial \mathcal{X}$, we have that $s(\x_i; \boldtheta) p(\x_i) \to \mathbf{0}$ for all $\boldtheta \in \boldsymbol{\Theta}$. 
\end{assumption}
The proof is based on integration by parts \cite{hyvarinen_2005_estimation}. The following unbiased estimator of \eqref{eq:sm_loss} is then obtained with the samples $\x_1^N$:
\begin{equation}
\hat{J}(\boldtheta; \x_1^N) \triangleq \frac{1}{N} \sum_{i=1}^N  \tr \left(\nabla_{\x} s(\x_i; \boldtheta) \right ) + \frac{1}{2} \| s(\x_i; \boldtheta)\|_2^2.
\label{eq:sm_loss_empirical}
\end{equation}
Note that computing \eqref{eq:sm_loss_empirical} only requires the evaluation of the vector-valued model and its derivative, and under additional regularity conditions, \eqref{eq:sm_loss_empirical} is a consistent estimator of \eqref{eq:sm_loss} \cite{hyvarinen_2005_estimation, song_2020_sliced}.

Since the introduction of score matching, a number of extensions and variants of the technique have been developed. Hv\"{a}rinen extended the approach to binary-valued variables and variables defined over bounded domains \cite{hyvarinen2007some}. Kingma and LeCun introduced a regularized version of score matching \cite{kingma2010regularized}. Song et al developed a scalable version of score matching, known as sliced score matching, for high-dimensional data by using projections to approximate the Jacobian term in the score matching objective \cite{song_2020_sliced}. They also showed that their approach was consistent, and that both the original and sliced score matching objectives lead to asymptotically normal estimators of the score. Song and Ermon's seminal generative modeling work \cite{song_2019_generative} then leveraged both sliced score matching and a denoising version of score matching introduced by Vincent \cite{vincent2011connection} as training objectives for diffusion models, which have achieved state-of-the-art results in generative modeling \cite{dhariwal_2021_diffusion}. 

Theoretical results in the works cited above are limited to asymptotic characterizations of score matching. To the best of our knowledge, the only previous non-asymptotic bounds on score matching were developed by Koehler et al in \cite{koehler2022statistical}, which considers the case where score matching is used to learn an energy-based model, i.e., a deep generative model parameterized up to the constant of parameterization. Koehler et al show that in this setting the expected  Kullback–Leibler (KL) divergence between the data distribution and the learned distribution can be bounded by a term proportional to the Rademacher complexity of the parameterized model class and the log-Sobolev constant, which relates the KL divergence to the score matching loss. This approach provides important insight into score matching performance. However, it does not attempt to bound the Rademacher complexity, and can therefore be viewed as somewhat orthogonal to our score matching analysis, which provides explicit error bounds. 

%% file: Sections/ProposedEstimator.tex
\section{Proposed Estimator}
\label{sec:proposed_approach}

Given i.i.d. samples $\{ \x_1, \x_2, \cdots, \x_N \} \subset \mathbb{R}^D$ from a prior distribution $p(\x)$ and a known likelihood function $p(\y | \x)$, our task is to obtain an estimate $\hat{\V}_B(\x_1^N)$ of the Bayesian CRB $\V_B$. Here we have assumed Assumptions \ref{assumption:support} - \ref{assumption:error_boundary} hold to make the problem well-defined. This problem is well motivated from a variety of applications (e.g., imaging or communications) in which we have complete knowledge of the measurement process (e.g., from physics), in addition to our prior information through some previously collected training data.

To address the problem, we propose a data-driven estimator of the Bayesian CRB using score matching. The proposed method makes use of the decomposition of $\J_B$ in \eqref{eq:BCRB_decomposition}, and estimates $\J_P$ and $\J_D$ separately. To estimate $\J_P$, we use the following estimator:
\begin{equation}
    \label{eq:Jp_estimator}
    \hat{\J}_P(\x_1^N) = \frac{1}{N} \sum_{i=1}^N s(\x_i; \hat{\boldtheta}_N) s(\x_i; \hat{\boldtheta}_N)^T,
\end{equation}
where $s(\x_i; \hat{\boldtheta}_N)$ is the prior score estimate obtained by minimizing the empirical score matching loss \eqref{eq:sm_loss_empirical}, and we have assumed Assumptions \ref{assumption:regularity} and \ref{assumption:boundary} hold. Note that in \eqref{eq:Jp_estimator}, we use the sample mean to approximate the expectation.

To estimate $\J_D$, we consider the following two cases for the known data model. First, if $\J_F(\x)$ can be computed analytically with the given data model (e.g., a linear/nonlinear Gaussian or Poisson model), we use the following estimator:
\begin{equation}
\label{eq:Jd_estimator1}
\hat{\J}_D(\x_1^N) = \frac{1}{N} \sum_{i=1}^N \J_F(\x_i).
\end{equation}
Note that this encompasses a linear Gaussian model as a special example, in which $\J_F(\x)$ is a constant independent of $\x$. Second, if $\J_F(\x)$ cannot be computed in a closed-form, we construct the following estimator:
\begin{equation}
\label{eq:Jd_estimator2}
\hat{\J}_D(\x_1^N) = \frac{1}{N} \sum_{i=1}^N \hat{\J}_F(\x_i),
\end{equation}
where
\begin{equation*}
\hat{\J}_F(\x_i) = \frac{1}{M}\sum_{j=1}^M \nabla_\x \log p(\y_j^{i} \, | \, \x_i) \nabla_\x \log p(\y_j^{i} \, | \, \x_i)^T.
\label{Jf_approximate}
\end{equation*}
Here we assume that we can obtain i.i.d. samples from the given data model (e.g., using Monte Carlo methods \cite{Christian1999}), i.e., $\y_j^{i} \sim p(\y \, | \, \x_i)$ for $j = 1, \cdots, M$.

Putting together the above estimators for $\J_P$ and $\J_D$, we form the following estimator for the Bayesian information:
\begin{equation*}
\hat{\J}_B(\x_1^N) = \hat{\J}_P(\x_1^N) + \hat{\J}_D(\x_1^N),
\end{equation*}
from which we obtain the Bayesian CRB estimator:
\begin{equation*}
    \hat{\V}_B(\x_1^N) = \hat{\J}_B(\x_1^N)^\dagger,
\end{equation*}
where the pseudoinversion ensures that $\hat{\V}_B(\x_1^N)$ is well defined.

%% file: Sections/classicalconverge.tex
\section{Convergence in the Classical Regime}
\label{sec:classicalconverge}

The goal of this section is to derive non-asymptotic error bounds for our Bayesian information and Bayesian CRB estimators in the classical parametric regime, where the dimension of the score model parameter space is less than the number of available samples from the prior distribution. To this end, we first obtain a generalization bound on the distance between the minimizer of the empirical score matching loss $\hat{\boldtheta}_N$ and the true minimizer $\boldtheta^* \triangleq \argmin_{\boldtheta \in \boldsymbol{\Theta}} L(\boldtheta)$ by leveraging existing score matching bounds and the rate theorem from empirical process theory \cite{wellner_1996_weak}. We then build off of this result and covariance matrix estimation bounds \cite{vershynin_2018_high} to obtain error bounds for the Bayesian CRB estimator.

\subsection{Score Matching Convergence Rate}

In this subsection, we provide convergence rate analysis for score matching, which we will use to prove error bounds for our Bayesian information and Bayesian CRB estimators. We first summarize key results from previous work on the consistency of the score matching estimator \cite{hyvarinen_2005_estimation, song_2020_sliced}. As in \cite{song_2020_sliced}, we require the following three assumptions to hold. 

\begin{assumption}[Compactness]
\label{assumption:compactness}
The parameter space $\boldsymbol{\Theta}$ is compact. 
\end{assumption}

\begin{assumption}[Lipschitz Continuity]
\label{assumption:lipschitz1}
Both $\nabla_{\x} s(\x; \boldtheta)$ and $s(\x; \boldtheta)s(\x; \boldtheta)^T$ are Lipschitz continuous in terms of the Frobenius norm, i.e., for any $\boldtheta_1, \boldtheta_2 \in \boldsymbol{\Theta}$, $\| \nabla_{\x} s(\x; \boldtheta_1) - \nabla_{\x} s(\x; \boldtheta_2) \|_2 \leq L_1(\x)\| \boldtheta_1 - \boldtheta_2 \|_2$ and $\| s(\x; \boldtheta_1)s(\x; \boldtheta_1)^T - s(\x; \boldtheta_2)s(\x; \boldtheta_2)^T \|_2 \leq L_2(\x)\| \boldtheta_1 - \boldtheta_2 \|_2$. Additionally, we require $\mathbb{E}_{\x}[ L_1(\x)^2 ] < \infty$ and $\mathbb{E}_{\x}[ L_2(\x)^2 ] < \infty$.
\end{assumption}

\begin{assumption}[Exact Minimization]
\label{assumption:exact_min1}
The optimized parameters $\hat{\boldtheta}_N$ exactly minimize the empirical loss \eqref{eq:sm_loss_empirical}. 
\end{assumption} 

From these assumptions, we obtain the following uniform bound on the expected error from the empirical approximation of the true score matching objective. The theorem is a straightforward modification of Lemma 3 in \cite{song_2020_sliced}, which proves a similar bound for sliced score matching. 
\begin{theorem}[Uniform convergence of the expected error] 
\label{lemma:uniform_converge}
Assume the score matching regularity conditions (Assumptions \ref{assumption:regularity} and \ref{assumption:boundary}) hold and Assumptions \ref{assumption:compactness} and \ref{assumption:lipschitz1} are satisfied. Then there exists a constant $C_S$
such that 
\begin{equation}
\mathbb{E}_{\x} \left[ \sup_{\boldtheta \in \boldsymbol{\Theta}} \left | \hat{J}(\boldtheta; \x_1^N) - J(\boldtheta) \right |  \right] \leq \frac{C_S}{\sqrt{N}}
\end{equation} 
for all $N$. 
\end{theorem}
\begin{proof}[Proof Sketch]
    We first show that $\ell(\x; \boldtheta) \triangleq \text{tr}\left(\nabla_{\x} s(\x; \boldtheta) \right ) + \frac{1}{2} \| s(\x; \boldtheta)\|_2^2$ is Lipschitz continuous in $\boldtheta$ (Lemma \ref{lemma:lipschitz}), which is an analogous result to Lemma 2 in \cite{song_2020_sliced}). Uniform convergence of the expected error then follows using standard techniques from empirical process theory (i.e., the symmetrization trick  and Dudley's entropy integral). See Appendix \ref{apppendix:classical} for the full proof. 
\end{proof}

In this work, we also require the following additional assumption, which ensures that the objective function is well behaved around $\boldtheta^*$. 
\begin{assumption} [Locally Quadratic]
\label{assumption:locally_quadratic}
There exist $\lambda, \eta > 0$ such that 
\begin{equation*}
J(\boldtheta) \geq J(\boldtheta^*) + \lambda \| \boldtheta - \boldtheta^* \|_2^2
\end{equation*}
for all $\boldtheta \in B_{\eta}(\boldtheta^*)$. Further, we have that 
\begin{equation*}
J(\boldtheta) \geq J(\boldtheta^*) + \lambda \eta^2
\end{equation*}
for all $\boldtheta \not \in B_{\eta}(\boldtheta^*)$.
\end{assumption}
Using the above assumption, we can bound the score matching error in parameter space. 

\begin{lemma}
\label{lemma_2}
Assume the score matching regularity conditions (Assumptions \ref{assumption:regularity} and \ref{assumption:boundary}) hold and Assumptions \ref{assumption:compactness}, \ref{assumption:lipschitz1}, and \ref{assumption:exact_min1} are satisfied. Then
    \begin{equation}
        \label{eq:lemma_2}
    \mathbb{P} \left [ \| \hat{\boldtheta}_N - \boldtheta^* \|_2 \geq \eta \right ] \leq \frac{2C_S}{\sqrt{N} \lambda \eta^2}. 
\end{equation}
\end{lemma}
\begin{proof}
 We have that 
    \begin{align*}
         \mathbb{E}_{\x} \left [ J(\hat{\boldtheta}_N) - J(\boldtheta^*) \right ] &= \mathbb{E}_{\x} \left [ \hat{J}(\hat{\boldtheta}_N; \x_1^N) + (J(\hat{\boldtheta}_N) - \hat{J}(\hat{\boldtheta}_N; \x_1^N)) - \hat{J}(\boldtheta^*; \x_1^N) - (J(\boldtheta^*) - \hat{J}(\boldtheta^*; \x_1^N)) \right] \nonumber \\
        &\numIneq{(i)} \mathbb{E}_{\x} \left [(J(\hat{\boldtheta}_N) - \hat{J}(\hat{\boldtheta}_N; \x_1^N)) -  (J(\hat{\boldtheta}_N) - \hat{J}(\hat{\boldtheta}_N; \x_1^N) \right ] \nonumber \\
        &\leq \mathbb{E}_{\x}\left [ | J(\hat{\boldtheta}_N) - \hat{J}(\hat{\boldtheta}_N; \x_1^N) | \right] + \mathbb{E}_{\x}\left [ | J(\boldtheta^*) - \hat{J}(\boldtheta^*; \x_1^N) | \right] \nonumber \\
        &\numIneq{(ii)} 2 \frac{C_S}{\sqrt{N}},
    \end{align*}
    where $(i)$ holds since $\hat{J}(\hat{\boldtheta}_N; \x_1^N) \leq \hat{J}(\boldtheta^*; \x_1^N)$ by definition of $\hat{\boldtheta}_N$ and $(ii)$ is by Theorem \ref{lemma:uniform_converge}. 
Applying Markov's inequality gives 
\begin{equation*}
    \mathbb{P} \left [ J(\hat{\boldtheta}_N) - J(\boldtheta^*) \geq \lambda \eta^2 \right ] \leq \frac{2C_S}{\sqrt{N} \lambda \eta^2}. 
\end{equation*}
By Assumption \ref{assumption:locally_quadratic}, this implies that 
\begin{equation*}
    \mathbb{P} \left [ \| \hat{\boldtheta}_N - \boldtheta^* \|_2 \geq \eta \right ] \leq \frac{2C_S}{\sqrt{N} \lambda \eta^2}
\end{equation*}
as desired. 
\end{proof}

After a change of variables, \eqref{eq:lemma_2} can be rewritten as 
$$
\mathbb{P} \left [ \| \hat{\boldtheta}_N - \boldtheta^* \|_2 \geq \frac{\sqrt{2 C_S}}{\sqrt{\lambda \epsilon} N^{1/4}} \right ] \leq \epsilon,
$$
so the convergence rate has $N^{1/4}$ dependence on $N$. In the following, we show that this rate can be improved upon. Our proof utilizes the following corollary of Theorem \ref{lemma:uniform_converge}. 
\begin{corollary}
\label{lemma:obj_diff_bound}
Assume all of the previously stated assumptions hold. Then there exists a $C_{\boldtheta} > 0$ such that for any $\delta > 0$, 
\begin{equation*}
\mathbb{E}_{\x} \left [\sup_{\|\boldtheta - \boldtheta^*\|_2 \leq \delta} \left | \Delta_N(\boldtheta) \right |  \right ] \leq \frac{C_{\boldtheta} \sqrt{P} \delta}{\sqrt{N}}
\end{equation*}
for all $N$, where $P$ is the dimension of the parameter space $\boldsymbol{\Theta}$, $C_{\boldtheta}$ is independent of $P$, and $$\Delta_N(\boldtheta) \triangleq (\hat{J}(\boldtheta; \x_1^N) - J(\boldtheta)) - (\hat{J}(\boldtheta^*; \x_1^N) - J(\boldtheta^*)).$$  
\end{corollary}
\begin{proof}
See Appendix \ref{apppendix:classical}.
\end{proof}

We are now ready to present our score matching convergence rate result. The proof is based on the rate of convergence theorem for empirical processes \cite[Theorem~3.2.5]{wellner_1996_weak}. 

\begin{theorem}
\label{theorem:finite_sample_bound}
Assume all of the previously stated assumptions hold. Then for any $\epsilon > 0$ and all $N \geq N' \triangleq 16 C_S^2/ \epsilon^2 \lambda^2 \eta^4$, with probability at least $1 - \epsilon$ it holds that
\begin{equation*}
\sqrt{N} \lVert \hat{\boldtheta}_N - \boldtheta^* \rVert_2 \leq \frac{8 C_{\boldtheta} \sqrt{P}}{\epsilon \lambda}.
\end{equation*}
\end{theorem}
\begin{proof}
See Appendix \ref{apppendix:classical}. 
\end{proof}

\subsection{Bayesian CRB Bounds}

In this subsection we prove non-asymptotic error bounds on the proposed estimators of the Bayesian information and Bayesian CRB using the just-proved score matching convergence rate. Our result requires an additional assumption, which ensures the scores of the prior distribution and likelihood function are well-behaved.

\begin{assumption}[Sub-Gaussian Scores]
\label{assumption:sub_gaussian}
The random vectors $\nabla_\x \log p(\x)$ and $\nabla_\x \log p(\y \, | \, \x)$ are sub-Gaussian with norms $C_P$ and $C_D$, respectively, i.e., for any $\mathbf{z} \in \mathbb{S}^{D-1}$ we have that 
\begin{equation*}
   \mathbb{E}_{\x} \left [ e^{\langle \nabla_\x \log p(\x), \mathbf{z} \rangle^2 / C_P^2 } \right ] \leq 2
\end{equation*}
and 
\begin{equation*}
   \mathbb{E}_{\x, \y} \left [ e^{ \langle \nabla_\x \log p(\y \mid\x), \mathbf{z} \rangle^2 / C_D^2  } \right ] \leq 2. 
\end{equation*}
\end{assumption}

Note that this assumption is satisfied for many Bayesian models of interest. For example, any prior distribution with bounded support (e.g., distributions of image data) will satisfy these assumptions if the score is a continuous function of $\x$, as the scores will be bounded and therefore sub-Gaussian. Further, inverse problems with Gaussian or Gaussian mixture priors and likelihood functions (see, e.g., \cite{flam2012mmse, stuart2010inverse}) also satisfy the above assumptions. The following example proves this is the case for a model with a Gaussian prior. 

\begin{example}[Gaussian prior]
\label{example:gaussian}
Consider a model with a Gaussian prior and without loss of generality assume that the prior is zero mean. Letting $\boldSigma$ denote its covariance matrix, for any $\mathbf{z} \in \mathbb{S}^{D-1}$ we have that
\begin{align*}
  \mathbb{E}_{\x} \left [e^{\langle \nabla_{\x} \log p(\x), \mathbf{z} \rangle^2 / C_P^2  } \right ]  &= \mathbb{E}_{\x} \left [e^{\langle -\boldSigma^{-1} \x, \mathbf{z} \rangle^2 / C_P^2  } \right ] \\
  &= C \int_{\x} e^{-\frac{1}{2} \x^T \boldSigma^{-1} \x} e^{\langle -\boldSigma^{-1} \x, \mathbf{z} \rangle^2 / C_P^2} \; d \x  \\
  &\leq C \int_{\x} e^{- \frac{1}{2} \x^T \boldSigma^{-1} \x} e^{\x^T \boldSigma^{-2} \x / C_P^2} \; d \x. 
\end{align*}
where $C$ is the normalizing constant for the Gaussian prior and the inequality holds because $\mathbf{z} \in \mathbb{S}^{D-1}$. Setting $C_P$ so that $C_P^2 = 2\| \boldSigma^{-1} \|_2 \beta$, $\beta > 1$, we have that 
$$
C \int_{\x} e^{- \frac{1}{2} \x^T \boldSigma^{-1} \x} e^{\x^T \boldSigma^{-2} \x / C_P^2} \; d \x \leq C \int_{\x} e^{- \frac{1}{2} \x^T \boldSigma^{-1} \x} e^{\x^T \boldSigma^{-1} \x / (2\beta)} \; d \x = C \int_{\x} e^{- \frac{\beta - 1}{2 \beta} \x^T \boldSigma^{-1} \x} \; d \x,
$$
which is clearly finite. Further scaling $C_P$ so that the integral is less than $2$ shows that $\nabla_{\x} \log p(\x)$ is sub-Gaussian. So Gaussian priors satisfy Assumption \ref{assumption:sub_gaussian}.
\end{example}

In addition to the above assumption, our results make use of the following two lemmas.

\begin{lemma}
\label{lemma:sub_gaussian}
Let $\x \in \mathbb{R}^D$ be a sub-Gaussian random vector, and define
\begin{equation*}
\hat{\boldsymbol{\Sigma}} \triangleq \frac{1}{N} \sum_{i=1}^N \x_i \x_i^T, \quad \boldsymbol{\Sigma} \triangleq \mathbb{E}_{\x} \left [ \x \x^T \right ].
\end{equation*}
Then for all $\epsilon > 0$, with probability at least $1 - \epsilon$ we have
\begin{equation*}
    \lVert \hat{\boldsymbol{\Sigma}} - \boldsymbol{\Sigma} \rVert_2 \leq C_{\boldsymbol{\Sigma}} \| \x \|^2_{\psi_2} m \left(\sqrt{\frac{D - \log(\epsilon)}{N}} \right),
\end{equation*} 
where $\| \cdot \|_{\psi_2}$ is the sub-Gaussian norm, $C_{\boldsymbol{\Sigma}}$ is a universal constant, and $m(t) \triangleq \mathrm{max} \{ t, t^2 \}$.
\end{lemma}
\begin{proof}
This is a minor modification of \cite[Proposition~2.1]{vershynin_2012_close}. 
\end{proof}

\begin{lemma}
\label{lemma:expected_outer_bound}
Assume all of the previously stated assumptions hold. Then
\begin{equation}
  \mathbb{E}_{\x} \left [ \| s(\x; \boldtheta^*)s(\x; \boldtheta^*)^T - \nabla_{\x} \log p(\x) \nabla_{\x} \log p(\x)^T \|_\sigma  \right ] \leq 2 L(\boldtheta^*) + 2 \mu_P \sqrt{2 L(\boldtheta^*) },
\end{equation}
where $\mu_P \triangleq \mathbb{E}_{\x} \left [\| \nabla_{\x} \log p(\x) \|_2^2 \right ]^{1/2}$.
\end{lemma}
\begin{proof}
We have that
\begin{align}
    \label{eq:covarianceScore_triangle}
    \mathbb{E}_{\x} &\left [ \| s(\x; \boldtheta^*)s(\x; \boldtheta^*)^T - \nabla_{\x} \log p(\x) \nabla_{\x} \log p(\x)^T \|_\sigma  \right ] \nonumber \\
    &\hspace{1cm} =  \mathbb{E}_{\x} \left [ \| s(\x; \boldtheta^*)s(\x; \boldtheta^*)^T - s(\x; \boldtheta^*) \nabla_{\x} \log p(\x)^T + s(\x; \boldtheta^*) \nabla_{\x} \log p(\x)^T - \nabla_{\x} \log p(\x) \nabla_{\x} \log p(\x)^T \|_\sigma  \right ] \nonumber  \\
    &\hspace{1cm}\leq \mathbb{E}_{\x} \left [\| s(\x; \boldtheta^*) \left ( s(\x; \boldtheta^*) - \nabla_{\x} \log p(\x) \right )^T \|_\sigma \right ] + \mathbb{E}_{\x} \left [\| \left ( s(\x; \boldtheta^*) - \nabla_{\x} \log p(\x) \right ) \nabla_{\x} \log p(\x)^T \|_\sigma \right ] \nonumber \\
    &\hspace{1cm}= \mathbb{E}_{\x} \left [\| \left ( s(\x; \boldtheta^*) - \nabla_{\x} \log p(\x) + \nabla_{\x} \log p(\x) \right ) \left ( s(\x; \boldtheta^*) - \nabla_{\x} \log p(\x) \right )^T \|_\sigma \right ] \nonumber \\
    &\hspace{1cm} \quad \quad + \mathbb{E}_{\x} \left [\| \left ( s(\x; \boldtheta^*) - \nabla_{\x} \log p(\x) \right ) \nabla_{\x} \log p(\x)^T \|_\sigma \right ] \nonumber \\
    &\hspace{1cm}\leq \mathbb{E}_{\x} \left [\| \left ( s(\x; \boldtheta^*) - \nabla_{\x} \log p(\x) \right ) \left ( s(\x; \boldtheta^*) - \nabla_{\x} \log p(\x) \right )^T \|_\sigma \right ] \nonumber \\
    &\hspace{1cm} \quad \quad + 2\mathbb{E}_{\x} \left [\|\left ( s(\x; \boldtheta^*) - \nabla_{\x} \log p(\x) \right ) \nabla_{\x} \log p(\x)^T \|_\sigma \right ] \nonumber \\
    &\hspace{1cm} \numIneq{(i)} \mathbb{E}_{\x} \left [ \| s(\x; \boldtheta^*) - \nabla_{\x} \log p(\x) \|_2^2 \right ] + 2\mathbb{E}_{\x} \left [ \| s(\x; \boldtheta^*) - \nabla_{\x} \log p(\x) \|_2 \| \nabla_{\x} \log p(\x) \|_2 \right ],
\end{align}
where $(i)$ holds by sub-multiplicity of the matrix norm. Now note that since $\J_P$ is well defined and $\mu_P = \sqrt{\text{tr}(\J_P)}$, $\mu_P$ is also well defined. We can therefore apply the Cauchy–Schwarz inequality to \eqref{eq:covarianceScore_triangle} to obtain 
\begin{align*}
  \mathbb{E}_{\x}& \left [ \| s(\x; \boldtheta^*)s(\x; \boldtheta^*)^T - \nabla_{\x} \log p(\x) \nabla_{\x} \log p(\x)^T \|_\sigma  \right ]  \nonumber \\
  &\hspace{1cm} \leq \mathbb{E}_{\x} \left [ \| s(\x; \boldtheta^*) - \nabla_{\x} \log p(\x) \|_2^2 \right ] +  2\mathbb{E}_{\x} \left [\| \nabla_{\x} \log p(\x) \|_2^2 \right ]^{1/2}  \mathbb{E}_{\x} \left [ \| s(\x; \boldtheta^*) - \nabla_{\x} \log p(\x) \|_2^2 \right ]^{1/2} \nonumber \\
  &\hspace{1cm} = 2 L(\boldtheta^*) + 2 \mu_P \sqrt{2 L(\boldtheta^*) },
\end{align*}
as desired.
\end{proof}

We now introduce the main results of this section. 

\begin{theorem} 
\label{theorem:bayesianInfo}
Assume all of the previously stated assumptions hold. Then for any $\epsilon > 0$ and any $N \geq N'$, the Bayesian information estimator satisfies, with probability at least $1 - \epsilon$,
\begin{equation}
\|\hat{\J}_B(\x_1^N) - \J_B\|_\sigma \leq C_1 \left [ (C_P^2 + C_D^2) \mathrm{m} \left (\sqrt{\frac{D - \log(\epsilon)}{N}} \right ) + \frac{C_{\boldtheta} \sqrt{P}}{\epsilon \lambda \sqrt{N}} \left ( \mu_L + \frac{\sigma_L}{\sqrt{\epsilon N}} \right) + \frac{1}{\epsilon} \left ( L(\boldtheta^*) + \mu_P \sqrt{L(\boldtheta^*)}\right ) \right ],
\label{Bound1}
\end{equation}
where $C_1$ is a universal constant, $\mu_L \triangleq \mathbb{E}_{\x} \left [ L_2(\x) \right]$, and $\sigma_L^2 \triangleq \mathbb{E}_{\x} \left [ (L_2(\x) - \mu_L)^2 \right]$.
\end{theorem}

\begin{proof}[Proof Sketch]
Through application of the triangle inequality, the Bayesian information estimation error can be bounded by the error in sample-based estimation of $\J_D$ and $\J_P$, a model mismatch term, and error in the score estimates. The sample based estimation error terms can be handled using Lemma \ref{lemma:sub_gaussian}, the model mismatch terms can be handled using Lemma \ref{lemma:expected_outer_bound} and Markov's inequality, and the score matching error term can be bounded using Theorem \ref{theorem:finite_sample_bound}. See Appendix \ref{apppendix:classical} for a full proof. 
\end{proof}

\begin{remark} A bound on the Bayesian information estimation error can still be proven if Assumption \ref{assumption:sub_gaussian} is weakened. For example, if the score vectors are sub-exponential, a weaker version of Theorem \ref{theorem:bayesianInfo} can be proven using covariance estimation bounds for sub-exponential random variables (see, e.g., \cite{vershynin_2012_close}). 
\end{remark}

\begin{theorem}
\label{theorem:BayesianCRB}
Assume all of the previously stated assumptions hold and that the model is well-specified, i.e., $L(\boldtheta^*) = 0$. Then there exists a constant $C_V$ such that for any $\epsilon > 0$ and any $N \geq C_V \mathrm{max} \{D - \log(\epsilon),  1/\epsilon^2 \}$, the Bayesian CRB estimator satisfies, with probability at least $1 - \epsilon$,
\begin{equation}
 \| \hat{\V}_B(\x_1^N) - \V_B \|_\sigma \leq C_2 \| \V_B \|_\sigma^2 \left [  (C_P^2 + C_D^2) \mathrm{m} \left (\sqrt{\frac{D - \log(\epsilon)}{N}} \right ) + \frac{C_{\boldtheta} \sqrt{P}}{\epsilon \lambda \sqrt{N}} \left ( \mu_L + \frac{\sigma_L}{\sqrt{\epsilon N}} \right) \right ],
\end{equation}
where $C_2$ is a universal constant. 
\end{theorem}

\begin{proof}
Assume that $\hat{\J}_B(\x_1^N)$ is invertible, which will be guaranteed later by concentration arguments. Conditioned on this event, $\hat{\V}_B(\x_1^N) - \V_B$ can be rewritten as 
\begin{align*}
\hat{\V}_B(\x_1^N) - \V_B &= \hat{\V}_B(\x_1^N)  ( \J_B - \hat{\J}_B(\x_1^N)) \V_B  \nonumber \\
&= (\hat{\V}_B(\x_1^N) - \V_B + \V_B)  ( \J_B - \hat{\J}_B(\x_1^N)) \V_B \nonumber \\
&= (\hat{\V}_B(\x_1^N) - \V_B)( \J_B - \hat{\J}_B(\x_1^N)) \V_B + \V_B( \J_B - \hat{\J}_B(\x_1^N)) \V_B. 
\end{align*}
Let $\Delta \triangleq \J_B - \hat{\J}_B(\x_1^N)$. Taking the norm of both sides of the above expression and applying the triangle inequality gives 
\begin{equation*}
   \| \hat{\V}_B(\x_1^N) - \V_B  \|_\sigma \leq  \| \hat{\V}_B(\x_1^N) - \V_B\|_\sigma \| \Delta \|_\sigma \| \V_B \|_\sigma + \| \V_B \|_\sigma^2 \| \Delta \|_\sigma,
\end{equation*}
which can be rewritten as 
\begin{equation*}
 \| \hat{\V}_B(\x_1^N) - \V_B  \|_\sigma (1 - \| \Delta \|_\sigma \| \V_B \|_\sigma) \leq \| \V_B \|_\sigma^2 \| \Delta \|_\sigma. 
 \end{equation*}
Note that by Theorem \ref{theorem:bayesianInfo} and the well-specified assumption, for $N \geq  C_V \mathrm{max} \{D - \log(\epsilon),  1/\epsilon^2 \}$, $C_V \geq  16 C_S^2/ \lambda^2 \eta^4$, it holds that 
\begin{equation*}
    \| \Delta \|_\sigma \leq C_1 \left [ (C_P^2 + C_D^2) \sqrt{\frac{1}{C_V}} + \frac{C_{\boldtheta} \sqrt{P}}{\lambda\sqrt{C_V}}\left (\mu_L + \frac{\sigma_L}{\sqrt{C_V}} \right ) \right ].
\end{equation*}
with probability at least $1 - \epsilon$. So we can choose $C_V$ such $\| \Delta \|_\sigma \leq 1/2\|\V_B\|_\sigma$ with probability at least $1 - \epsilon$. In this regime $\hat{\J}_B(\x_1^N)$ is guaranteed to be invertible, and it holds that 
\begin{align*}
   \| \hat{\V}_B(\x_1^N) - \V_B  \|_\sigma &\leq (1 - \| \Delta \|_\sigma \| \V_B \|_\sigma)^{-1}  \| \V_B \|_\sigma^2 \| \Delta \|_\sigma \\ 
   &\leq 2 \| \V_B \|_\sigma^2 \| \Delta \|_\sigma.
\end{align*}
Applying Theorem \ref{theorem:bayesianInfo} with $L(\boldtheta^*) = 0$ to $\Delta$, taking the union bound of the above probabilities, and introducing the universal constant gives the desired result.   
\end{proof}

\begin{remark}
The above theorems establish non-asymptotic bounds for our Bayesian information and CRB estimators. These results rely on a couple of key assumptions, including a locally quadratic objective function (Assumption \ref{assumption:locally_quadratic}) and sub-Gaussian score vectors (Assumption \ref{assumption:sub_gaussian}). However, the consistency of our estimators, i.e., their asymptotic convergence, can be proven without these assumptions. See our early work for the theorem and proof \cite{crafts2023}.
\end{remark}

%% file: Sections/convergeNN.tex
\section{Convergence with Neural Network Models}
\label{sec:convergeNN}

The Bayesian information and Bayesian CRB bounds in the previous section have $\sqrt{P}$ dependence on the dimension $P$ of the parameter space $\boldsymbol{\Theta}$ and $1/{\sqrt{N}}$ dependence on the sample size $N$. Since modern neural networks are often highly overparameterized, i.e., $N \ll P$, these bounds are inadequate for cases where the score model is a neural network. In this section, we address this limitation by developing bounds for our estimator with neural network score models that have improved dependence on the parameter space dimension. Specifically, we make the following assumption about the form of $s(\x; \boldtheta)$.

\begin{assumption}[Model Structure]
\label{assumption:model}
The parametric model $s(\x; \boldtheta)$ is a feedforward neural network that can be written as follows:
$$
s(\x; \boldtheta) = \sigma_L (\W_L \sigma_{L-1} ( \cdots \sigma_1(\W_1 \x))).
$$
Here $\boldtheta \triangleq \{ \W_i \}_{i=1}^L$ are the neural network weights, i.e., $\W_i \in \mathbb{R}^{d_i \times d_{i-1}}$ with $d_0 = d_L = D$. The $\{ \sigma_{i} \}_{i=1}^L$ are fixed nonlinearities $\sigma_i: \mathbb{R}^{d_i} \to \mathbb{R}^{d_i}$. We assume that they satisfy the following properties:
\begin{enumerate}
    \item They are Lipschitz with Lipschitz constants $\{\rho_i\}^{L}_{i=1}$.
    \item They satisfy $\sigma_i(0) = 0$.
    \item Evaluation of the derivatives of the is $\tau_i$-Lipschitz i.e., $$\left \lVert \,  \nabla_{\x} \sigma_i(\x) |_{\mathbf{s}} - \nabla_{\x} \sigma_i(\x) |_{\mathbf{t}} \,  \right \rVert_2 \leq \tau_i \| \mathbf{s} - \mathbf{t} \|_2$$ for any $\mathbf{s}, \mathbf{t} \in \mathbb{R}^{d_i}$.
    \item The Jacobians are bounded by constants $f_i$ in the spectral norm, i.e.,  $ \| \, \nabla_{\x} \sigma_i(\x) |_{\mathbf{s}}  \, \|_{\sigma} \leq f_i$ for any $\mathbf{s} \in \mathbb{R}^{d_i}$. 
\end{enumerate}
\end{assumption}

Note that many commonly used nonlinearities, such as pointwise Tanh or Softplus functions, satisfy the above conditions.  We also make two additional assumptions about the model and data distribution. 
\begin{assumption}[Bounded Support]
\label{assumption:bounded_support}
The data distribution has bounded support, i.e., there exists a constant $T$ such that $\| \x \|_2 \leq T$ for all $\x \in \mathcal{X}$. 
\end{assumption}

\begin{assumption}[Bounded Weights]
\label{assumption:bounded_weights}
The model weights $\W_i$ lie in spaces $\mathcal{W}_i$ that satisfy $\| \W_i \|_{\sigma} \leq c_i$ and $\|\W_i \|_{2, 1} \leq b_i$ for all $\W_i \in \mathcal{W}_i$ and all $i$. The parameter space $\boldsymbol{\Theta}$ therefore denotes the Cartesian product space $\mathcal{W}_1 \times \mathcal{W}_2 \times \dots \times \mathcal{W}_L$. 
\end{assumption}

Finally, as in Section \ref{sec:classicalconverge}, we make an assumption regarding the optimized model parameters.

\begin{assumption}[Neural Network Optimization]
\label{assumption:exact_min2}
The optimized parameters $\hat{\boldtheta}_N$ satisfy $\hat{J}(\hat{\boldtheta}_N; \x_1^N) \leq \hat{J}(\boldtheta^*; \x_1^N)$.
\end{assumption}

\subsection{Score Matching Convergence Rate}

This subsection provides a bound on $L(\hat{\boldtheta}_N)$, the score matching error, under the above assumptions. To that end, we first show that  $L(\hat{\boldtheta}_N)$ can be related to the empirical Rademacher complexity of the neural network function class. Our result makes use of the following two lemmas. 

\begin{lemma}[Theorem 2.1 in \cite{bartlett_2005_local}]
\label{lemma:empirical_bound}
Let $\mathcal{F}$ be a class of functions that maps $\mathcal{X}$ into $[-M, M]$. Assume that there is some $r \geq 0$ such that for every $f \in \mathcal{F}$, $\var [ f(\x) ] \leq r$. Then for every $\epsilon > 0$, with probability at least $1 - \epsilon$ over the data $\mathbf{X} = [\x_1, \ldots, \x_N]$, we have that 
$$
\sup_{f \in \mathcal{F}} \mathbb{E}_{\x} f(\x) - \frac{1}{N} \sum_{i=1}^N f(\x_i) \leq 6 \mathfrak{R}\left ( \mathcal{F}|_{\mathbf{X}} \right ) + \sqrt{\frac{2 r \log(2 / \epsilon)}{N}} + \frac{32 M \log(2 / \epsilon) }{3N},
$$
where $\mathcal{F}|_{\mathbf{X}} \triangleq \{ [f(\x_1), \dots, f(\x_N)] \mid f \in \mathcal{F} \}$ and $\mathfrak{R}\left (\mathcal{F}|_{\mathbf{X}} \right )$ is the empirical Rademacher complexity of $\mathcal{F}$, i.e., 
$$
\mathfrak{R}\left (\mathcal{F}|_{\mathbf{X}} \right ) \triangleq \frac{1}{N}  \mathbb{E}_{\epsilon} \left [\sup_{f \in \mathcal{F}} \sum_{i=1}^N \epsilon_i f(\x_i) \right]
$$
where the $\epsilon_i$ are independent Rademacher random variables.  
\end{lemma}

\begin{lemma}
\label{lemma:nn_model_bound}
Assume Assumptions \ref{assumption:model}, \ref{assumption:bounded_support}, and \ref{assumption:bounded_weights} are satisfied. Then for any $\x \in \mathcal{X}$ and any $\boldtheta \in \boldsymbol{\Theta}$, 
$$
| \ell(\x; \boldtheta) | \leq B, \quad B \triangleq \frac{T^2}{2} \prod_{i=1}^L \rho_i^2 c_i^2 + \prod_{i=1}^L b_i f_i.
$$
\end{lemma}
\begin{proof}
    See Appendix \ref{appendix:nn}.
\end{proof}

We now introduce the first major result of this section. 

\begin{theorem}
\label{theorem:nn_reduction2rademacher}
    Assume Assumptions \ref{assumption:model}-\ref{assumption:exact_min2} and the regularity conditions in Section \ref{sec:back_prelim} hold. Then 
    \begin{equation}
\label{eq:G_empirical_bound}
L(\hat{\boldtheta}_N) \leq L(\boldtheta^*)  + \mathfrak{R}\left (\mathcal{G}|_{\mathbf{X}} \right ) + \sqrt{\frac{8 B^2 \log(2/\epsilon)}{N}} + \frac{64 B \log(2 / \epsilon) }{3N}
\end{equation}
with probability at least $1 - \epsilon$ over $\{ \x_i \}_{i=1}^N$, where $\mathcal{G} \triangleq \{\ell(\cdot; \boldtheta) - \ell(\cdot; \boldtheta^*) \mid \boldtheta \in \boldsymbol{\Theta} \} $ and, as in the previous section, $\boldtheta^* \triangleq \argmin_{\boldtheta \in \boldsymbol{\Theta}} L(\boldtheta)$. 
\end{theorem}

\begin{proof}
First note that by Assumption \ref{assumption:exact_min2}, $\hat{J}(\hat{\boldtheta}_N; \x_1^N) \leq \hat{J}(\boldtheta^*; \x_1^N)$, so we have that 
\begin{align}
L(\hat{\boldtheta}_N) &= J(\hat{\boldtheta}_N) + C \nonumber \\
&\leq J(\hat{\boldtheta}_N)- \hat{J}(\hat{\boldtheta}_N; \x_1^N) + \hat{J}(\boldtheta^*; \x_1^N)  + C  \nonumber \\
&= J(\hat{\boldtheta}_N)- \hat{J}(\hat{\boldtheta}_N; \x_1^N) + \hat{J}(\boldtheta^*; \x_1^N) - J(\boldtheta^*) + J(\boldtheta^*)  + C \nonumber \\
&= J(\hat{\boldtheta}_N)- \hat{J}(\hat{\boldtheta}_N; \x_1^N) + \hat{J}(\boldtheta^*; \x_1^N) - J(\boldtheta^*) + L(\boldtheta^*).
\label{eq:initialBound}
\end{align}
The term $L(\boldtheta^*)$ quantifies the model mismatch. The remaining terms define an empirical process
$$
J(\hat{\boldtheta}_N)- \hat{J}(\hat{\boldtheta}_N; \x_1^N) + \hat{J}(\boldtheta^*; \x_1^N) - J(\boldtheta^*) = \mathbb{E}_{\x} \left(\ell(\cdot; \hat{\boldtheta}_N) - \ell(\cdot; \boldtheta^*)\right) - \frac{1}{N} \sum_{i=1}^N \left(\ell(\x_i; \hat{\boldtheta}_N) - \ell(\x_i; \boldtheta^*)\right),
$$
indexed by $\hat{\boldtheta}_N$ over the function class $\mathcal{G}$. Note that for any $g \in \mathcal{G}$ and $\x \in \mathcal{X}$, by Lemma \ref{lemma:nn_model_bound} we have that 
\begin{equation*}
\label{eq:gx_bound2}
|g(\x)| \leq 2  \sup_{\boldtheta \in \boldsymbol{\Theta}} \left|\ell(\x; \boldtheta)  \right | \leq 2 B.
\end{equation*}
Further it holds that 
$$
\var(g(\x)) \leq \sup_{\x_1, \x_2 \in \mathcal{X}}  \frac{1}{4} |g(\x_1) - g(\x_2) |^2 \leq \sup_{\x \in \mathcal{X}} | g(\x) |^2 \leq 4 B^2.
$$
Applying Lemma \ref{lemma:empirical_bound} with $r = 4 B^2$ and $M = 2B$ to the function class $\mathcal{G}$ therefore gives
\begin{equation*}
\label{eq:empirical_bound_app}
 \mathbb{E}_{\x} \left(\ell(\cdot; \hat{\boldtheta}_N) - \ell(\cdot; \boldtheta^*)\right) - \frac{1}{N} \sum_{i=1}^N \left(\ell(\x_i; \hat{\boldtheta}_N) - \ell(\x_i; \boldtheta^*)\right) \leq \mathfrak{R}\left (\mathcal{G}|_{\mathbf{X}} \right ) + \sqrt{\frac{8 B^2 \log(2 / \epsilon)}{N}} + \frac{64 B \log(2 / \epsilon) }{3N}
\end{equation*}
with probability at least $1 - \epsilon$ over $\{ \x_i \}_{i=1}^N$. Incorporating the above result into \eqref{eq:initialBound} completes the proof. 
\end{proof}

The above theorem reduces the problem to bounding the empirical Rademacher complexity of $\mathcal{G}$. The following lemma, which is a straightforward generalization of Lemma A.5 in \cite{bartlett2017spectrally}, can be used to relate the empirical Rademacher complexity to the covering number of the function class. 

\begin{lemma}[Lemma A.5 in \cite{bartlett2017spectrally}]
\label{lemma:dudley_rademacher}
Let $\mathcal{F}$ be a real valued function class taking values in $[-M, M]$ and assume that $\mathbf{0} \in \mathcal{F}$. Then 
    $$
    \mathfrak{R}\left (\mathcal{F}|_{\mathbf{X}} \right ) \leq \inf_{\alpha > 0} \left (\frac{4\alpha}{\sqrt{N}} + \frac{12}{N} \int_{\alpha}^{M \sqrt{N}} \sqrt{\log \mathcal{N}(\mathcal{F}|_{\mathbf{X}}, \epsilon, \| \cdot \|_2)} \; d \epsilon \right ).
    $$
\end{lemma}

The following two lemmas from \cite{bartlett2017spectrally} can be used to bound the covering number of neural networks. 

\begin{lemma}[Theorem 3.3 in \cite{bartlett2017spectrally}]
\label{lemma:network_cover}
Assume Assumptions \ref{assumption:model}, \ref{assumption:bounded_support}, and \ref{assumption:bounded_weights} are satisfied. For data $\{\x_{i} \}_{i=1}^N$, define 
$$
\mathcal{H} \triangleq \{ [s(\x_1; \boldtheta), s(\x_2; \boldtheta), \dots, s(\x_N; \boldtheta)] \mid \boldtheta \in \boldsymbol{\Theta} \}. 
$$
Then for any $\epsilon > 0$, 
$$
\log \mathcal{N}(\mathcal{H}, \epsilon, \| \cdot \|_2) \leq \frac{ \| \mathbf{X}\|_2^2 \log (2d_{\mathrm{max}}^2)}{\epsilon^2} \left ( \prod_{j=1}^L c_j^2 \rho_j^2 \right ) \left ( \sum_{i=1}^L \left (\frac{b_i}{c_i} \right )^{2/3} \right )^3,
$$
where $d_{\mathrm{max}} = \max \{ d_0, \dots d_L \}$.
\end{lemma}

\begin{lemma}[Lemma 3.2 in \cite{bartlett2017spectrally}] 
\label{lemma:bartlett3.2}
Let conjugate exponents $(p, q)$ and $(r, s)$ be given with $p \leq 2$m as well as positive reals $(a, b, \epsilon)$ and positive integer $m$. Let matrix $\mathbf{X} \in \mathbb{R}^{n \times d}$ be given with $\| \mathbf{X} \|_p \leq b$, where $\| \mathbf{X} \|_p$ is the element-wise p-norm. Then 
$$
\log \mathcal{N}\left ( \left \{\mathbf{X} \mathbf{A} \mid \mathbf{A} \in \mathbb{R}^{d \times m}, \| \mathbf{A} \|_{q, s} \leq a \right \}, \epsilon, \| \cdot \|_2 \right ) \leq \left \lceil \frac{a^2 b^2 m^{2/r} }{\epsilon^2 } \right \rceil \log(2 d m),
$$
where $\lceil \cdot \rceil$ is the ceiling operator. 
\end{lemma}

We also make use of the following result, which we use to handle the covering number of sets of matrix-matrix products. 

\begin{lemma}
\label{lemma:product_cover}
Suppose that we are given a collection of sets $\{ \mathcal{Y}_l \}_{l=1}^L$, where each set $\mathcal{Y}_l$ contains tensors $\mathbf{Y}^l \in \mathbb{R}^{N \times d_l \times d_{l-1}}$ and satisfies
$$
\mathcal{N}(\mathcal{Y}_l, \epsilon_l, \| \cdot \|_2) \leq v_l
$$
for some $\epsilon_l$ and $v_l$. Define $\mathbf{Y} \triangleq \mathbf{Y}^L \mathbf{Y}^{L-1} \cdots \mathbf{Y}^1$, where for two tensors $\mathbf{A} \in \mathbb{R}^{N \times d_0 \times d_1}$ and $\mathbf{B} \in \mathbb{R}^{N \times d_1 \times d_2}$, the bilinear operation $\mathbf{A} \mathbf{B}$ yields a tensor $\mathbf{C} \in \mathbb{R}^{N \times d_0 \times d_2}$ defined by
$$
\mathbf{C}_{i} = \mathbf{A}_{i} \mathbf{B}_{i}
$$
for $i = 1, \cdots, N$, where $\mathbf{A}_i$ is shorthand for the matrix $\mathbf{A}_{i, :, :}$. Further, assume that for any $l$, every element $\mathbf{Y}^l$ of either $\mathcal{Y}_l$ or the cover of $\mathcal{Y}_l$ satisfies $\| \mathbf{Y}^l_i \|_2 \leq b_l$ for any $i$. Then 
for
$$
\epsilon = \sum_{l=1}^L \epsilon_i \prod_{k \neq l} b_k
$$
we have that 
$$
\mathcal{N}(\mathcal{Y}, \epsilon, \| \cdot \|_2 ) \leq \prod_{l=1}^L v_l,
$$
where $\mathcal{Y} = \{ \mathbf{Y} \mid \mathbf{Y} = \mathbf{Y}^L \mathbf{Y}^{L-1} \cdots \mathbf{Y}^1, \mathbf{Y}^l \in \mathcal{Y}_l \text{  for  }l=1, \dots, L\}$. 
\end{lemma}

\begin{proof}
    See Appendix \ref{appendix:nn}.
\end{proof}

We are now ready to bound the empirical Rademacher complexity of $\mathcal{G}|_{\mathbf{X}}$. 

\begin{theorem}
\label{theorem:boundingRademacherNN}
    Assume Assumptions \ref{assumption:model}-\ref{assumption:exact_min2}, \ref{assumption:regularity} and \ref{assumption:boundary} hold. Then the empirical Rademacher complexity of the function class $\mathcal{G}$ satisfies

    $$
    \mathfrak{R}\left (\mathcal{G}|_{\mathbf{X}} \right ) \leq  \frac{12 \sqrt{R}}{N} \left (1 +  \log(2 B N / 3 \sqrt{R} ) \right ),
    $$
    where 
    $$
R = 16 L D \bar{\alpha}^2 \log (2d_{\mathrm{max}}^2) \| \mathbf{X}\|_2^2   \prod_{l = 1}^L f_l^2 b_l^2  + 4 T^2 \| \mathbf{X}\|_2^2  \log (2d_{\mathrm{max}}^2) \left ( \prod_{j=1}^L c_j^4 \rho_j^4 \right ) \left ( \sum_{i=1}^L \left (\frac{b_i}{c_i} \right )^{2/3} \right )^3 
$$
with 
$$
\bar{\alpha} = \sum_{l=1}^L \left ( \prod_{j=1}^{l-1} c_j \rho_j \right ) c_l \tau_l \left (c_l^2  \left ( \sum_{i=1}^l \left (\frac{b_i}{c_i} \right )^{2/3} \right )^3  + b_l^2 \right )^{1/2}.
$$ 
\end{theorem}

\begin{proof}

First note that the assumption that the zero function lies in the function class is trivially satisfied for $\mathcal{G}$, and that $|g(\x)| < 2B$ for any $g \in \mathcal{G}$. So using Lemma \ref{lemma:dudley_rademacher}, we obtain 
\begin{equation}
\label{eq:RademacherCoveringRelation}
\mathfrak{R}\left (\mathcal{G}|_{\mathbf{X}} \right ) \leq \inf_{\alpha > 0} \left (\frac{4\alpha}{\sqrt{N}} + \frac{12}{N} \int_{\alpha}^{2 B \sqrt{N}} \sqrt{\log \mathcal{N}(\mathcal{G}|_{\mathbf{X}}, \epsilon, \| \cdot \|_2)} \; d \epsilon \right ).
\end{equation}
Next, note that shifting by the fixed function $\ell(\cdot; \boldtheta^*)$ will not effect the covering number, so it is sufficient to bound the covering number of $\mathcal{G}'|_{\mathbf{X}}$, where $\mathcal{G}' \triangleq \{ \ell(\cdot; \boldtheta) \mid \boldtheta \in \boldsymbol{\Theta} \}$. Further, since $\ell(\cdot; \boldtheta) = \frac{1}{2} \| s(\x; \boldtheta)\|_2^2 + \text{tr}\left(\nabla_{\x} s(\x; \boldtheta) \right )$, we have that 
\begin{equation}
\label{eq:G1G2decomp}
\log \mathcal{N}(\mathcal{G}|_{\mathbf{X}}, 2\epsilon, \| \cdot \|_2) = \log \mathcal{N}(\mathcal{G'}|_{\mathbf{X}}, 2\epsilon, \| \cdot \|_2) \leq \log \mathcal{N}(\mathcal{G}_1|_{\mathbf{X}}, \epsilon, \| \cdot \|_2) + \log \mathcal{N}(\mathcal{G}_2|_{\mathbf{X}}, \epsilon, \| \cdot \|_2),
\end{equation}
where $\mathcal{G}_1 = \{\| s(\x; \boldtheta)\|_2^2 / 2 \mid \boldtheta \in \boldsymbol{\Theta} \}$ and $\mathcal{G}_2 = \{  \text{tr}\left(\nabla_{\x} s(\x; \boldtheta) \right )  \mid \boldtheta \in \boldsymbol{\Theta} \}$. 

We first bound the covering number of $\mathcal{G}_1$. To that end, note that for any $M$-Lipschitz function $f$ and any set $A$, we have that $\mathcal{N}(f(A), M\epsilon, \| \cdot \|_2) \leq \mathcal{N}(A, \epsilon, \| \cdot \|_2)$. Since the function $x^2/2$ is $M$-Lipschitz for $x \leq M$, we can use the bound on $\|s(\x; \boldtheta)\|_2$ given by \eqref{eq:twoNorm_bound} together with Lemma \ref{lemma:network_cover} to obtain that for any $\epsilon > 0$,  
$$
\log \mathcal{N}\left(\mathcal{G}_1 |_{\mathbf{X}},  \left (\prod_{i=1}^L \rho_i c_i \right ) T \epsilon, \| \cdot \|_2 \right) \leq  \frac{\| \mathbf{X}\|_2^2 \log (2d_{\mathrm{max}}^2)}{\epsilon^2} \left ( \prod_{j=1}^L c_j^2 \rho_j^2 \right ) \left ( \sum_{i=1}^L \left (\frac{b_i}{c_i} \right )^{2/3} \right )^3,
$$
which after a change of variables becomes 
\begin{equation}
\label{eq:G1bound}
\log \mathcal{N}(\mathcal{G}_1 |_{\mathbf{X}}, \epsilon, \| \cdot \|_2) \leq  \frac{T^2 \| \mathbf{X}\|_2^2 \log (2d_{\mathrm{max}}^2)}{\epsilon^2} \left ( \prod_{j=1}^L c_j^4 \rho_j^4 \right ) \left ( \sum_{i=1}^L \left (\frac{b_i}{c_i} \right )^{2/3} \right )^3.
\end{equation}

We now cover $\mathcal{G}_2$. We first define 
$$
\mathcal{H}_l \triangleq \{ [s_l(\x_1; \boldtheta), s_l(\x_2; \boldtheta), \dots, s_l(\x_N; \boldtheta)] \mid \boldtheta \in \boldsymbol{\Theta} \}
$$
for $l = 1, \dots, L$, where $s_l(\x; \boldtheta) \triangleq \W_l \sigma_{l-1} ( \cdots \sigma_1 (\W_1 \x))$. By straightforward modification of Lemma \ref{lemma:network_cover}, we have that for any $\epsilon > 0$, it holds that 
\begin{equation}
\log \mathcal{N}(\mathcal{H}_l, \epsilon, \| \cdot \|_2) \leq \frac{\| \mathbf{X}\|_2^2 \log (2d_{\mathrm{max}}^2)}{\epsilon^2} \left ( \prod_{j=1}^{l-1} c_j^2 \rho_j^2 \right ) c_l^2 \left ( \sum_{i=1}^l \left (\frac{b_i}{c_i} \right )^{2/3} \right )^3. \label{eq:partial_net_bound}
\end{equation}
Now let
$$
\mathcal{F}_l \triangleq \left \{ \left [  \nabla_{\x} \sigma_l(\x) \bigg \rvert_{s_l(\x_1; \boldtheta)}, \dots, \nabla_{\x} \sigma_l(\x) \bigg \rvert_{s_l(\x_N; \boldtheta)} \right] \mid \boldtheta \in \boldsymbol{\Theta} \right \},
$$
By Assumption \ref{assumption:model}, evaluation of $\nabla_{\x} \sigma_l(\x)$ at $s_l(\x; \boldtheta)$ is Lipschitz with Lipschitz constant $\tau_l$. Using the previously discussed Lipschitz property of covering numbers we can build off of the bound given in Eq. \eqref{eq:partial_net_bound} to obtain
\begin{equation*}
\label{eq:perTerm_gradientCover}
\log \mathcal{N}(\mathcal{F}_l, \tau_l \epsilon, \| \cdot \|_{2}) \leq \frac{ \| \mathbf{X}\|_2^2 \log (2d_{\mathrm{max}}^2)}{\epsilon^2} \left ( \prod_{j=1}^{l-1} c_j^2 \rho_j^2 \right ) c_l^2  \left ( \sum_{i=1}^l \left (\frac{b_i}{c_i} \right )^{2/3} \right )^3,
\end{equation*}
which after a change of variables becomes
\begin{equation}
\label{eq:perTerm_gradientCover2}
\log \mathcal{N}(\mathcal{F}_l, \epsilon, \| \cdot \|_{2}) \leq \frac{ \| \mathbf{X}\|_2^2 \log (2d_{\mathrm{max}}^2)}{\epsilon^2} \left ( \prod_{j=1}^{l-1} c_j^2 \rho_j^2 \right ) c_l^2 \tau_l^2 \left ( \sum_{i=1}^l \left (\frac{b_i}{c_i} \right )^{2/3} \right )^3.
\end{equation} 

We now extend the above result to bound the covering number of 
\begin{equation}
\label{eq:FlWL_term}
\mathcal{F}_l \mathcal{W}_l \triangleq \left \{ \left [  \nabla_{\x} \sigma_l(\x) \bigg \rvert_{s_l(\x_1; \boldtheta)} \W_l, \dots, \nabla_{\x} \sigma_l(\x) \bigg \rvert_{s_l(\x_N; \boldtheta)} \W_l \right] \mid \boldtheta \in \boldsymbol{\Theta}, \W_l \in \mathcal{W}_l \right \}.
\end{equation}
To this end, note that the covering number of this set under the componentwise two-norm is equivalent to the covering number of the set containing elements of the form 
$$
\mathbf{F}_l \W_l, \quad \mathbf{F}_l \triangleq \begin{bmatrix} \nabla_{\x} \sigma_l(\x) \bigg \rvert_{s_l(\x_1; \boldtheta)}\\
\vdots \\
\nabla_{\x} \sigma_l(\x) \bigg \rvert_{s_l(\x_N; \boldtheta)}
\end{bmatrix},
$$
which is the product of a matrix of size $d_i N \times d_i$ with a matrix of size $d_i \times d_{i-1}$. By  Assumption \ref{assumption:bounded_weights}, we have that $\| \W_l \|_{2, 1} \leq b_l$, and by Assumptions \ref{assumption:model} and \ref{assumption:bounded_weights}, we have that $$\| \mathbf{F}_l \|_2 \leq  \| \mathbf{X}\|_2 \left ( \prod_{i=1}^{l-1} c_i \rho_i \right ) c_l \tau_l. $$
Applying Lemma \ref{lemma:bartlett3.2} with $p = q = 2$ and $r = \infty$, $s = 1$ to the matrix product $\mathbf{F}_l \W_l$ therefore gives 
$$
\log \mathcal{N}\left ( \left \{\mathbf{F}_l \mathbf{W}_l \mid \mathbf{W}_l \in \mathcal{W}_l \right \}, \epsilon, \| \cdot \|_2 \right ) \leq \left \lceil \frac{\| \mathbf{X}\|_2^2} {\epsilon^2 } \left ( \prod_{i=1}^{l-1} c_i^2 \rho_i^2 \right ) b_l^2 c_l^2 \tau_l^2
\right \rceil \log(2 d_{\text{max}}^2).
$$
The above result bounds the covering number of $\mathbf{F}_l \W_l$ under the assumption that $\mathbf{F}_l$ is fixed. However, we can incorporate the covering result from \eqref{eq:perTerm_gradientCover2} into this result to bound the covering number of \eqref{eq:FlWL_term}. Specifically, recalling the Lipschitz property of covering numbers and noting that by Assumption \ref{assumption:bounded_weights}, we have that $|| \W_l ||_{\sigma} \leq c_l$, it holds that  
\begin{align*}
\log \mathcal{N} \left (\mathcal{F}_l \mathcal{W}_l, \epsilon_1 + \epsilon_2 c_l, \| \cdot \|_2 \right ) &\leq \frac{ \| \mathbf{X}\|_2^2 \log (2d_{\mathrm{max}}^2)}{\epsilon_2^2} \left ( \prod_{j=1}^{l-1} c_j^2 \rho_j^2 \right ) c_l^2 \tau_l^2 \left ( \sum_{i=1}^l \left (\frac{b_i}{c_i} \right )^{2/3} \right )^3 \\
& \quad + \frac{\| \mathbf{X}\|_2^2  \log(2 d_{\text{max}}^2)} {\epsilon_1^2 } \left ( \prod_{j=1}^{l-1} c_j^2 \rho_j^2 \right ) b_l^2 c_l^2 \tau_l^2.
\end{align*}
Setting $\epsilon_1 = \epsilon / 2$ and $\epsilon_2 c_l = \epsilon / 2$ and combining like terms gives 
\begin{equation}
    \label{eq:FlWlbound}
    \log \mathcal{N} \left (\mathcal{F}_l \mathcal{W}_l, \epsilon, \| \cdot \|_2 \right ) \leq \frac{4   \| \mathbf{X}\|_2^2 \log (2d_{\mathrm{max}}^2) }{\epsilon^2} \left ( \prod_{j=1}^{l-1} c_j^2 \rho_j^2 \right ) c_l^2 \tau_l^2 \left (c_l^2  \left ( \sum_{i=1}^l \left (\frac{b_i}{c_i} \right )^{2/3} \right )^3  + b_l^2 \right ).
\end{equation}

We now have a bound on the covering number of $ \mathcal{F}_l \mathcal{W}_l$ for $l = 1, \dots L$. What remains is to bound the covering number of the product of these terms, i.e., to bound the covering number of $\nabla_{\x} s(\x; \boldtheta)$. To that end, identify $\mathcal{Y}_l$ in Lemma \ref{lemma:product_cover} with $\mathcal{F}_l \mathcal{W}_l$ and $v_l$ with the covering number bound given by Eq. \eqref{eq:FlWlbound}. Also note that by Assumptions \ref{assumption:model} and \ref{assumption:bounded_weights}, we have that $\|( \mathbf{F}_l \W_l)_{i, :, :} \|_2 \leq f_l b_l$ for any $\mathbf{F}_l \W_l \in \mathcal{F}_l \mathcal{W}_l$ and any $i \in 1, \dots, N$, where here we view $\mathbf{F}_l \W_l$ as a $N \times d_i \times d_{i-1}$ tensor. We argue that any element of the cover of $\mathcal{F}_l \mathcal{W}_l$ can also be made to satisfy this bound. To see this, note that if there is an element of the cover that does not satisfy this bound, we can simply replace it by its projection onto the set of terms satisfying the bound while maintaining the epsilon-cover. So applying Lemma \ref{lemma:product_cover} gives

\begin{align}
        \log \mathcal{N} \left (\nabla_{\x} s(\x; \boldtheta), \sum_{l=1}^L \epsilon_l \prod_{k \neq l} f_k b_k, \| \cdot \|_2 \right ) \leq  
        \sum_{l=1}^L \frac{4  \log (2d_{\mathrm{max}}^2) \| \mathbf{X}\|_2^2 }{\epsilon_l^2} \underbrace{\left ( \prod_{j=1}^{l-1} c_j^2 \rho_j^2 \right ) c_l^2 \tau_l^2 \left (c_l^2  \left ( \sum_{i=1}^l \left (\frac{b_i}{c_i} \right )^{2/3} \right )^3  + b_l^2 \right )}_{\alpha_l}.
\label{eq:nablasx_initial}
\end{align}

Let $$\epsilon_l \triangleq \frac{\epsilon \sqrt{\alpha_l}}{\bar{\alpha} \prod_{k \neq l} f_k b_k}, \quad \text{where} \quad \bar{\alpha} \triangleq \sum_{l=1}^L \sqrt{\alpha_l} .$$
Then 
$$
\sum_{l=1}^L \epsilon_l \prod_{k \neq l} f_k b_k = \sum_{l=1}^L \frac{\epsilon \sqrt{\alpha_l}}{\bar{\alpha}} = \frac{\epsilon}{\bar{\alpha}} \sum_{l=1}^L \sqrt{\alpha_l} = \epsilon.
$$
Incorporating this change of variables into \eqref{eq:nablasx_initial} then gives 
\begin{equation*}
      \log \mathcal{N} \left (\nabla_{\x} s(\x; \boldtheta), \epsilon, \| \cdot \|_2 \right ) \leq   \frac{4  \bar{\alpha}^2 \| \mathbf{X}\|_2^2 \log (2d_{\mathrm{max}}^2) }{\epsilon^2} \sum_{l=1}^L \prod_{k \neq l} f_k^2 b_k^2 \leq \frac{4 L  \bar{\alpha}^2 \| \mathbf{X}\|_2^2 \log (2d_{\mathrm{max}}^2) }{\epsilon^2} \prod_{l=1}^L f_l^2 b_l^2.
\end{equation*}
Noting that the trace operator is $\sqrt{D}$-Lipschitz in the Frobenius norm and using the Lipschitz property of covering numbers, a bound on the covering number of $\mathcal{G}_2|_{\mathbf{X}}$ then immediately follows:
\begin{equation*}
\label{eq:G2bound}
      \log \mathcal{N} \left (\mathcal{G}_2|_{\mathbf{X}}, \epsilon, \| \cdot \|_2 \right ) \leq   \frac{4 L D \bar{\alpha}^2 \| \mathbf{X}\|_2^2 \log (2d_{\mathrm{max}}^2)}{\epsilon^2} \prod_{l = 1}^L f_l^2 b_l^2.
\end{equation*}
Incorporating this result and the bound on the covering number of $\mathcal{G}_1|_{\mathbf{X}}$ given by \eqref{eq:G1bound} into \eqref{eq:G1G2decomp} and applying a change of variables to $\epsilon$ then gives 
\begin{align}
\log \mathcal{N}(\mathcal{G}|_{\mathbf{X}}, \epsilon, \| \cdot \|_2) \leq & \; \frac{16 L D \bar{\alpha}^2 \| \mathbf{X}\|_2^2 \log (2d_{\mathrm{max}}^2)}{\epsilon^2} \prod_{l = 1}^L f_l^2 b_l^2 \nonumber \\
& + \frac{4 T^2 \| \mathbf{X}\|_2^2  \log (2d_{\mathrm{max}}^2)}{\epsilon^2} \left ( \prod_{j=1}^L c_j^4 \rho_j^4 \right ) \left ( \sum_{i=1}^L \left (\frac{b_i}{c_i} \right )^{2/3} \right )^3 \nonumber \\
\triangleq & \frac{R}{\epsilon^2}.
\label{eq:Gx_bound_final}
\end{align}
Incorporating the score matching covering bound given by \eqref{eq:Gx_bound_final} into the bound on $\mathfrak{R}\left (\mathcal{G}|_{\mathbf{X}} \right)$ given by \eqref{eq:RademacherCoveringRelation}, we then have that
$$
    \mathfrak{R}\left (\mathcal{G}|_{\mathbf{X}} \right ) \leq \inf_{\alpha > 0} \left (\frac{4\alpha}{\sqrt{N}} + \frac{12}{N} \int_{\alpha}^{2 B \sqrt{N}} \sqrt{\frac{R}{\epsilon^2}} \; d \epsilon \right ).
$$
This bound is uniquely minimized at $\alpha = 3 \sqrt{R/N}$. Plugging this value in to the above equation and simplifying gives the stated result. 

\end{proof}

A bound on $L(\hat{\boldtheta}_N)$ now follows from Theorems \ref{theorem:nn_reduction2rademacher} and \ref{theorem:boundingRademacherNN}. 

\begin{theorem}
\label{theorem:mainNNscoreBound}
    Assume Assumptions \ref{assumption:model}-\ref{assumption:exact_min2} and regularity conditions \ref{assumption:regularity} and \ref{assumption:boundary} hold. Then with probability $1 - \epsilon$ over the $\{ \x_i \}_{i=1}^N$, the score matching loss satisfies 
    \begin{equation}
    \label{eq:MainBoundwithCoveringNumber}
    L(\hat{\boldtheta}_N) \leq L(\boldtheta^*) + \sqrt{\frac{8 B^2 \log(2 / \epsilon) }{N}} + \frac{64 B \log(2/\epsilon) }{3N} + \frac{12 \sqrt{R}}{N} \left (1 +  \log(2 B N / 3 \sqrt{R} ) \right ),
\end{equation}
where $L(\boldtheta^*)$ quantifies the model mismatch and the remaining terms bound the generalization error. 
\end{theorem}

\begin{proof}
    This follows directly from incorporating the bound on the empirical Rademacher complexity given by Theorem \ref{theorem:boundingRademacherNN} into Theorem \ref{theorem:nn_reduction2rademacher}. 
\end{proof}

\subsection{Bayesian CRB Bounds}

In this subsection we build off the results introduced in the previous subsection to prove non-asymptotic error bounds on the proposed estimators of the Bayesian information and Bayesian CRB in the neural network model setting. As in the classical setting (Section \ref{sec:classicalconverge}), we require the scores of the likelihood function to be well-behaved, which is formalized through the following assumption. 

\begin{assumption}[Sub-Gaussian Scores]
\label{assumption:sub_gaussian_nns}
The random vector $\nabla_\x \log p(\y \, | \, \x)$ is sub-Gaussian with norms $C_D$, i.e., for any $\mathbf{z} \in \mathbb{S}^{D-1}$ we have that 
\begin{equation}
   \mathbb{E}_{\x, \y} \left [ e^{ \langle \nabla_\x \log p(\y \mid\x), \mathbf{z} \rangle^2 / C_D^2  } \right ] \leq 2. 
\end{equation}
\end{assumption}
Note that unlike the previous section, we do not also need to require the scores of the prior to be sub-Gaussian, as the assumption that the prior distribution has bounded support (Assumption \ref{assumption:bounded_support}) together with Assumption \ref{assumption:regularity} imply that the prior distribution scores are bounded and therefore sub-Gaussian. As in the previous section, we use $C_P$ to denote the sub-Gaussian norm of the prior distribution scores. 

Our bound also makes use of the following corollary of Lemma \ref{lemma:expected_outer_bound}. 

\begin{corollary}
\label{corollary:expected_outer_bound}
Assume the regularity conditions in Section \ref{sec:back_prelim} are satisfied. Then it holds that 
\begin{equation}
  \mathbb{E}_{\x} \left [ || s(\x; \hat{\boldtheta}_N)s(\x; \hat{\boldtheta}_N)^T - \nabla_{\x} \log p(\x) \nabla_{\x} \log p(\x)^T ||_\sigma  \right ] \leq 2 L(\hat{\boldtheta}_N) + 2 \mu_P \sqrt{2 L(\hat{\boldtheta}_N) },
\end{equation}
where $\mu_P \triangleq \mathbb{E}_{\x} \left [|| \nabla_{\x} \log p(\x) ||_2^2 \right ]^{1/2}$.
\end{corollary}
\begin{proof}
This is a straightforward extension of Lemma \ref{lemma:expected_outer_bound} and the proof is thus omitted. 
\end{proof}
We now introduce the main results. 

\begin{theorem} 
\label{theorem:bayesianInfo2}
Assume Assumptions \ref{assumption:model}-\ref{assumption:sub_gaussian_nns} hold and the regularity conditions in Section \ref{sec:back_prelim} are satisfied. Then for any $\epsilon > 0$ and all $N$, the Bayesian information estimator satisfies, with probability at least $1 - \epsilon$,
\begin{align}
||\hat{\J}_B(\x_1^N) - \J_B||_\sigma &\leq C_3 \Big [ (C_P^2 + C_D^2) \mathrm{m} \left (\sqrt{\frac{D - \log(\epsilon)}{N}} \right ) + \frac{1}{\epsilon} \left ( L(\boldtheta^*) + \mu_P \sqrt{L(\boldtheta^*)}\right ) \nonumber \\
& \quad + \frac{1}{\epsilon N} \left (B\log(1/\epsilon) + \sqrt{R}(1 + \log(BN/\sqrt{R})\right) +  \frac{1}{\epsilon N^{1/4}}\left ( \mu_P \sqrt{B} \log(1/\epsilon)^{1/4} \right) \nonumber \\
&\quad + \frac{1}{\epsilon N^{1/2}} \left ( B \log(1/\epsilon) + \mu_P \sqrt{B \log(1/\epsilon)} + \mu_P R^{1/4}(1 + \log(BN/\sqrt{R}))^{1/2} \right)
\Big ],
\label{Bound1_NN}
\end{align}
where $C_3$ is a universal constant.
\end{theorem}

\begin{proof}[Proof Sketch]
As in the proof of Theorem \ref{theorem:bayesianInfo}, the triangle inequality is used to bound the Bayesian information error, which reduces the problem to bounding the by error in the sample-based estimation of $\J_D$ and $\J_P$, as well as a score matching error term. We again use Lemma \ref{lemma:sub_gaussian} to bound the sample-based estimation of $\J_D$ and $\J_P$, while the score matching error is bounded using Theorem \ref{theorem:mainNNscoreBound} and Corollary \ref{corollary:expected_outer_bound}.

\end{proof}

At this point it is worth making a couple of remarks regarding the major differences between Theorem \ref{theorem:bayesianInfo} and Theorem \ref{theorem:bayesianInfo2}. Specifically, while both theorems provide the same dependence on the score matching model mismatch and the finite-sample covariance matrix estimation, their dependence on the score matching generalization error differs. Theorem \ref{theorem:bayesianInfo} leverages the local quadratic assumption (Assumption \ref{assumption:locally_quadratic}) to provide $1/\sqrt{N}$ dependence on the number of samples $N$ and $\sqrt{P}$ dependence on the number of model parameters $P$ once the estimated score parameters $\hat{\boldtheta}_N$ enter the locally quadratic neighborhood of $\boldtheta^*$. In contrast, while Theorem \ref{theorem:bayesianInfo2} provides worse $1/N^{1/4}$ dependence on $N$, the bound has $\sqrt{L}$ dependence on the model depth but only log dependence on the network width. As a consequence, in the infinite-width limit the bound has only log dependence on the number of model parameters, allowing the model depth to grow polynomially and the width to grow exponentially with the number of samples without comprising the bound. 

We now introduce the final result, which is a straightforward consequence of Theorem \ref{theorem:bayesianInfo2} and the proof techniques used in Theorem \ref{theorem:BayesianCRB}.

\begin{theorem}
\label{theorem:BayesianCRB2}
Assume Assumptions \ref{assumption:model}-\ref{assumption:sub_gaussian_nns} and the regularity conditions in Section \ref{sec:back_prelim} hold and that the model is well-specified, i.e., $L(\boldtheta^*) = 0$. Then there exists a constant $C_{V'}$ such that for any $\epsilon > 0$ and any $N \geq C_{V'} \mathrm{max} \{D - \log(\epsilon),  1/\epsilon^4 \}$, the Bayesian CRB estimator satisfies, with probability at least $1 - \epsilon$,
\begin{align}
||\hat{\V}_B(\x_1^N) - \V_B||_\sigma &\leq C_4 ||\V_B ||_\sigma^2 \Big [ (C_P^2 + C_D^2) \mathrm{m} \left (\sqrt{\frac{D - \log(\epsilon)}{N}} \right ) + \frac{1}{\epsilon} \left ( L(\boldtheta^*) + \mu_P \sqrt{L(\boldtheta^*)}\right ) \nonumber \\
& \quad + \frac{1}{\epsilon N} \left (B\log(1/\epsilon) + \sqrt{R}(1 + \log(BN/\sqrt{R})\right) +  \frac{1}{\epsilon N^{1/4}}\left ( \mu_P \sqrt{B} \log(1/\epsilon)^{1/4} \right) \nonumber \\
&\quad + \frac{1}{\epsilon N^{1/2}} \left ( B \log(1/\epsilon) + \mu_P \sqrt{B \log(1/\epsilon)} + \mu_P R^{1/4}(1 + \log(BN/\sqrt{R}))^{1/2} \right)
\Big ],
\end{align}
where $C_4$ is a universal constant. 
\end{theorem}

\begin{proof}
The proof is similar to that of Theorem \ref{theorem:BayesianCRB} and is thus omitted. 
\end{proof}

%% file: Sections/NumExperiments.tex
\section{Numerical Experiments}
\label{sec:num_experiments}

\subsection{Signal Denoising}

We first illustrate the performance of the proposed approach on the following simple denoising example:
$$
\y = \x + \mathbf{z}, \quad \mathbf{z} \sim N(\mathbf{0}, \tau^2 \mathbf{I}).
$$
Here $\x$ is a ten-dimensional random vector with the following Gaussian mixture prior distribution:
$
p(\x) = .4 p_1(\x) + .3 p_2(\x) + .3 p_3(\x),
$
where $p_1(\x)$ has mean $[-5, \cdots, -5]^T$ and an identity covariance matrix, $p_2(\x)$ has mean $[0, \cdots, 0]^T$ and a diagonal covariance matrix whose diagonal entries are linearly spaced between $1$ and $2$, and $p_3(x)$ has mean $[5, \cdots, 5]^T$ and a covariance matrix with the same eigenvalues as those of $p_2(\x)$ but randomly chosen eigenvectors.

\begin{figure*}
    \centering
    \includegraphics[width = \textwidth]{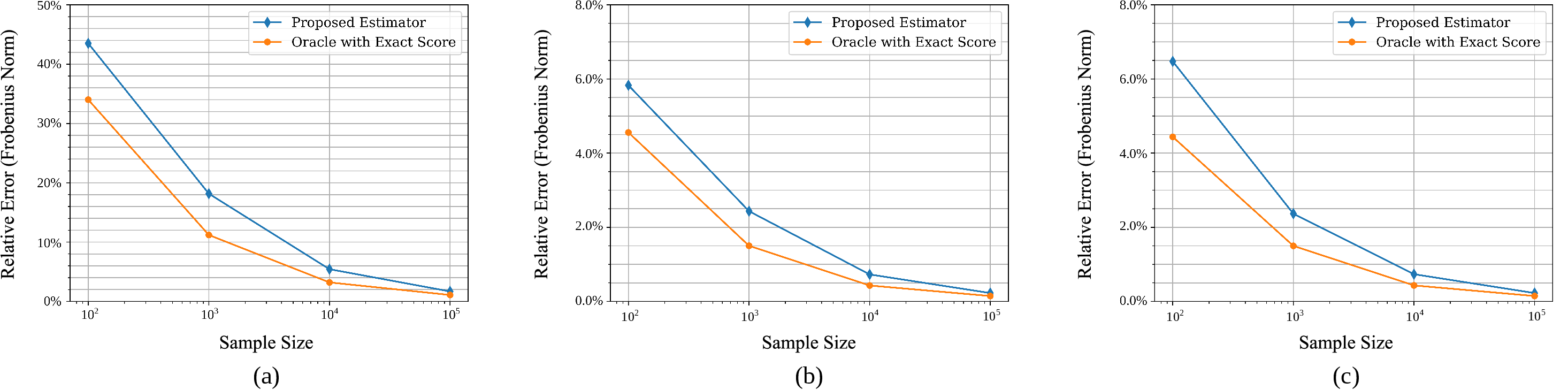}
    \caption{ Relative error in estimation of (a) the prior Fisher information $\J_P$, (b) the Bayesian information $\J_B$, and (c) the Bayesian CRB $\V_B$ as
a function of sample size. The
performance of an oracle that has access to the ground-truth prior score $\nabla_{\x} \log p(\x)$ is included as a reference at different sample sizes. When the dataset size is small, the error in $\J_P$ estimation is high, but the error decreases monotonically with $N$ and comparison with the oracle error shows most of the error is attributable to Monte-Carlo estimation of the score covariance matrix, not the score estimation. Further, the error in the Bayesian CRB estimation is under $5\%$ for all dataset sizes.}
    \label{fig:n_vs_error}
\end{figure*}
\begin{figure}
    \centering
    \includegraphics[width = .4\columnwidth]{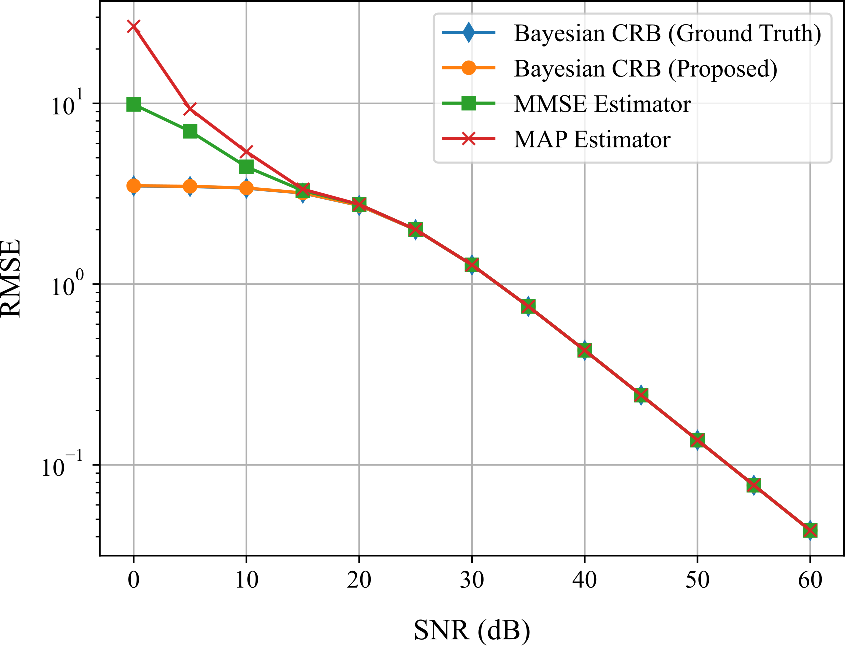}
    \caption{Root-mean-square error (RMSE) as a function of signal-to-noise ratio (SNR) for the ground-truth Bayesian CRB, the proposed estimator, and the minimum mean square error (MMSE) and maximum a posterior (MAP) estimators. Here the proposed estimator was implemented with $N = 10^4$ samples from the prior distribution, and for each SNR level $20{,}000$ samples from $p(\x, \y)$ were used to calculate the MAP and MMSE estimators' RMSEs. The Bayesian CRB estimator is visually identical to the ground truth Bayesian CRB and provides a tight lower bound on the MSE in the high SNR regime. }
    \label{fig:snr_vs_rmse}
\end{figure}

Note that this problem has a linear Gaussian data model, for which $\J_D$ can be computed analytically, i.e., $\J_D = (1/\tau^2)\mathbf{I}$. Here we focus on the proposed estimator for $\J_P$. Specifically, we examined the performance of the proposed approach as a function of the number of available prior samples $N$, with $N = 10^2, 10^3, 10^4$, or $10^5$.

For each $N$, the score-estimator was implemented using a vector-valued fully connected neural network with Softplus activations. Here the network had five hidden layers and width $1000$ for the $N = 10^2$ and $N=10^3$ cases, while five hidden layers were used for the $N = 10^4$ and $N=10^5$ cases, with a network width of  $50$ (the $N = 10^4$ case) or $200$ (the $N = 10^5$ case). Note that this model satisfies the assumptions in Section \ref{sec:convergeNN} regarding the network structure. During network training, the given dataset of $N$ samples was split into training and validation sets; here the validation set consisted of $5000$ samples ($N=10^5$ case), $1000$ samples ($N=10^4$ case), or $80\%$ of the total dataset ($N=10^2$ and $N=10^3$ cases). The networks were trained using the Adam optimizer \cite{kingma2015} with a learning rate of $10^{-4}$ (the $N = 10^4$ case) or $10^{-5}$ (all other cases); the batch size was $8000$ ($N = 10^5$ and $N=10^4$ cases) or set equal to the size of the training set ($N = 10^2$ and $N=10^3$ cases). Training was terminated when the score matching loss on this validation set had not improved in $200$ iterations; the network weights corresponding to the best validation loss were saved. 

For this problem, we have access to the ground truth score $\nabla_\x \log p(\x)$. Thus, we can calculate $\J_P$ using the ground truth scores, where the only source of estimation error is from Monte-Carlo sampling, and use this oracle estimator as a reference to assess the sources of error in the proposed estimator $\hat{\J}_P$. Together with $\J_D$, we can also provide references for the Bayesian information and Bayesian CRB estimators. Fig.~\ref{fig:n_vs_error} shows the relative errors in the estimation of $\J_P$, $\J_B$, and $\V_B$ at different sample sizes with signal-to-noise ratio (SNR) $= 10 \log_{10} \left ( \mathbb{E}_{\x} [ || \x ||_2^2 ] / \tau^2 \right ) = 30 \, \mathrm{dB}$. Here we treated the oracle with $10^{6}$ samples as the ground truth. As can be seen, the proposed method provides comparable performance with respect to the oracle. In particular, they both attain under $5\%$ relative error in the Bayesian information and Bayesian CRB estimation for all $N$. 

To examine the effectiveness of the proposed method as a lower bound on the MSE of estimators, we also implemented the minimum mean square error (MMSE) and maximum a posterior (MAP) estimators with the true prior distribution. Note that for the above denoising problem, it can be shown that the posterior distribution $p(\x|\y)$ is a Gaussian mixture, from which we can compute the MMSE estimator analytically. The MAP estimator was obtained by running gradient ascent on the posterior distribution. Fig.~\ref{fig:snr_vs_rmse} shows the root-mean-square error (RMSE) as a function of the SNR level for the ground-truth Bayesian CRB and the estimated Bayesian CRB, as well as the MAP and MMSE estimators. As can be seen, the proposed Bayesian CRB estimator is rather close to the ground-truth Bayesian CRB, which provides a tight lower bound on the MSE in the high-SNR regime.

\subsection{Phase Offset Estimation}

Dynamical phase offset estimation problem arises in digital receiver synchronization, which has applications in the design and analysis of communication systems. Here we consider the problem formulation described in \cite{bay2007analytic}. In this formulation, the received signal $\y = [y_1, \dots, y_D]^T$ is modeled by
$$
y_d = a_d e^{j x_d} + z_d, \quad d = 1, \dots, D,
$$
where $j$ is the imaginary unit, $x_d$ is the signal phase, $a_d$ is the transmitted symbol, which is modeled as a Rademacher random variable, and $z_d$ is zero-mean circular Gaussian noise with known variance $\tau_n^2$. The goal of this problem is to estimate $\x = [x_1, \dots, x_D]^T$ from the received data $\y$, with the symbols $a_d$ marginalized out in the likelihood definition $p(\y \mid \x)$. For the prior distribution, we adopt the prior considered in \cite{bay2007analytic} as the ground truth to enable quantification of the error in our estimator. This prior is given by a Wiener phase-offset evolution, i.e., 
$$
x_d  = x_{d-1} + w_d, \quad d = 2, \dots, D.
$$
the $w_d$ is i.i.d. Gaussian noise, i.e., $w_d \sim \mathcal{N}(0, \tau_w^2)$, and in this work we assume $x_0 \sim \mathcal{N}(0, 1)$ and set $D = 10$ and $\tau_w = .2$. See \cite{bay2007analytic} for more details of the problem setting.

\begin{figure}
    \centering
    \includegraphics[width = \textwidth]{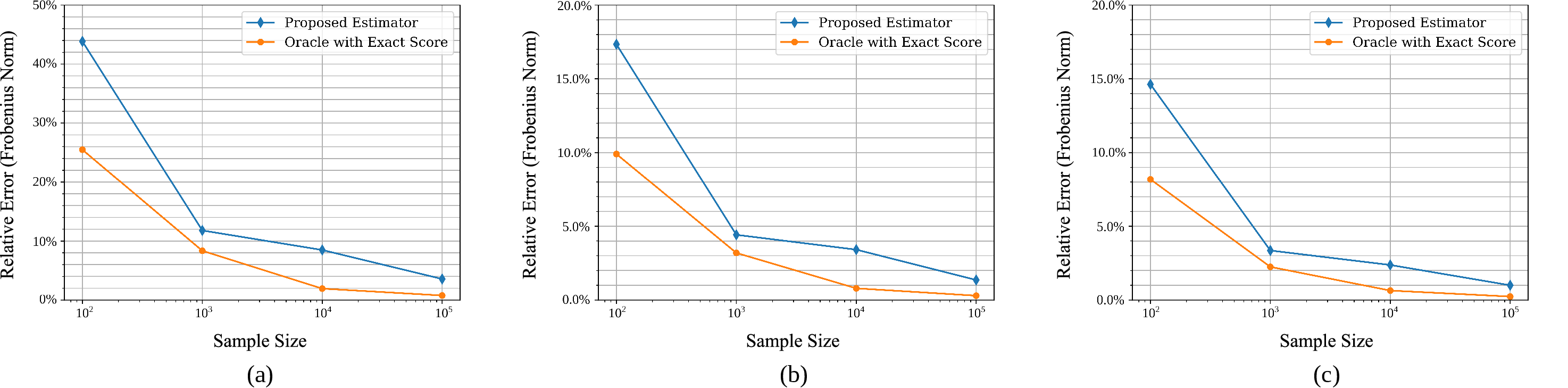}
    \caption{Relative error in the estimation of (a) the prior term $\J_P$, (b) the Bayesian information $\J_B$, and (c) the Bayesian CRB $\V_B$ as a function of the dataset size ($\tau_n^2 = .02$, SNR $\approx \mathrm 13 \; \mathrm{dB}$) for the phase offset Bayesian inverse problem. As expected, the error in the proposed approach is higher for small sizes, but comparison with the oracle error shows this mostly attributable to Monte-Carlo estimation of the second moment of the score, not the score estimation. Further, the error decreases monotonically with $N$, and the error in estimation of the Bayesian information and Bayesian CRB is under $5\%$ for $N \geq 10^3$.}
    \label{fig:rev_phasen}
\end{figure}

As in the signal denoising problem example, we first examine the performance of our estimator as a function of the dataset size $N$ with $N=10^2$, $10^3$, $10^4$, or $10^5$. For this problem, the score estimator in the proposed approach was implemented as a vector-valued fully-connected neural network with two hidden layers, Softplus activations, and network width of $1000$ for all $N$ (note that this is the same architecture as used in the signal denoising problem for $N = 10^2$ and $10^3$). The Adam optimizer \cite{kingma2015} was used with a learning rate of $10^{-5}$; the training and validation set sizes and the batch sizes for each $N$ were set as in the signal denoising example. The prior term estimate $\hat{J}_P$ was then formed from the score estimate using \eqref{eq:Jp_estimator}. To estimate $\J_D$, we take advantage of the special form $\J_D$ for this problem (see Eq. 14 in \cite{bay2007analytic}): 
$$
\J_D = J_D \mathbf{I}, \quad J_D \triangleq \mathbb{E}_{\y, \x} \left [ - \frac{\partial^2 \log p(y_d \mid x_d)}{\partial x_d^2} \right ].
$$
Specifically, we construct the following Monte-Carlo estimator of $\J_D$ using the given dataset $\x_1^N$:
\begin{equation}
\label{eq:jd_phase}
\hat{\J}_D = \hat{J}_D \mathbf{I}, \quad \hat{J}_D \triangleq \frac{1}{MN} \sum_{j=1}^{M} \sum_{i=1}^N \left [ - \frac{\partial^2 \log p(y_d^{ij} \mid x_d^{i})}{\partial x_d^2} \right ],
\end{equation}
where the $\y^{ij}$ are i.i.d. samples from the known likelihood function $p(\y \mid \x_i)$, the second derivative of the log-likelihood is given by Eq. 32 in \cite{bay2007analytic}, and here we set $M = 10$ for all $N$. 

Fig. \ref{fig:rev_phasen} shows the relative error of the proposed approach in estimation of $\J_P$, $\J_B$, and $\V_B$ as a function of the dataset size ($\tau_n = \sqrt{.2}$, SNR $ \approx 13 \; \mathrm{dB}$). As can be seen, the performance of the proposed approach for this problem is similar to its performance for the mixture of Gaussians denoising example, illustrating the robustness of the proposed approach to the prior and likelihood function of the given problem. In particular, the relative error decreases monotonically with $N$ and achieves a relative error in estimation of $\J_B$ and $\V_B$ of under $5 \%$ for $N \geq 10^3$. 

For this problem, we also examined the performance of the proposed approach for different choices of SNR.  Fig. \ref{fig:revsnr_phase} shows the relative error in estimation of the Bayesian information and Bayesian CRB for the proposed approach with SNRs in the range $[-10, 60] \; \mathrm{dB}$, where here the proposed approach was implemented with $N = 10^4$ samples using the setup described above.  As can be seen, the proposed approach provides a relative error in the Bayesian CRB estimation under $5\%$ at all SNR levels, with the highest error in the low SNR regime where the prior-informed term $\J_P$ dominates. This further demonstrates the robustness of the proposed approach to the given problem setting. \\

\begin{figure}
    \centering
    \includegraphics[width = .66\textwidth]{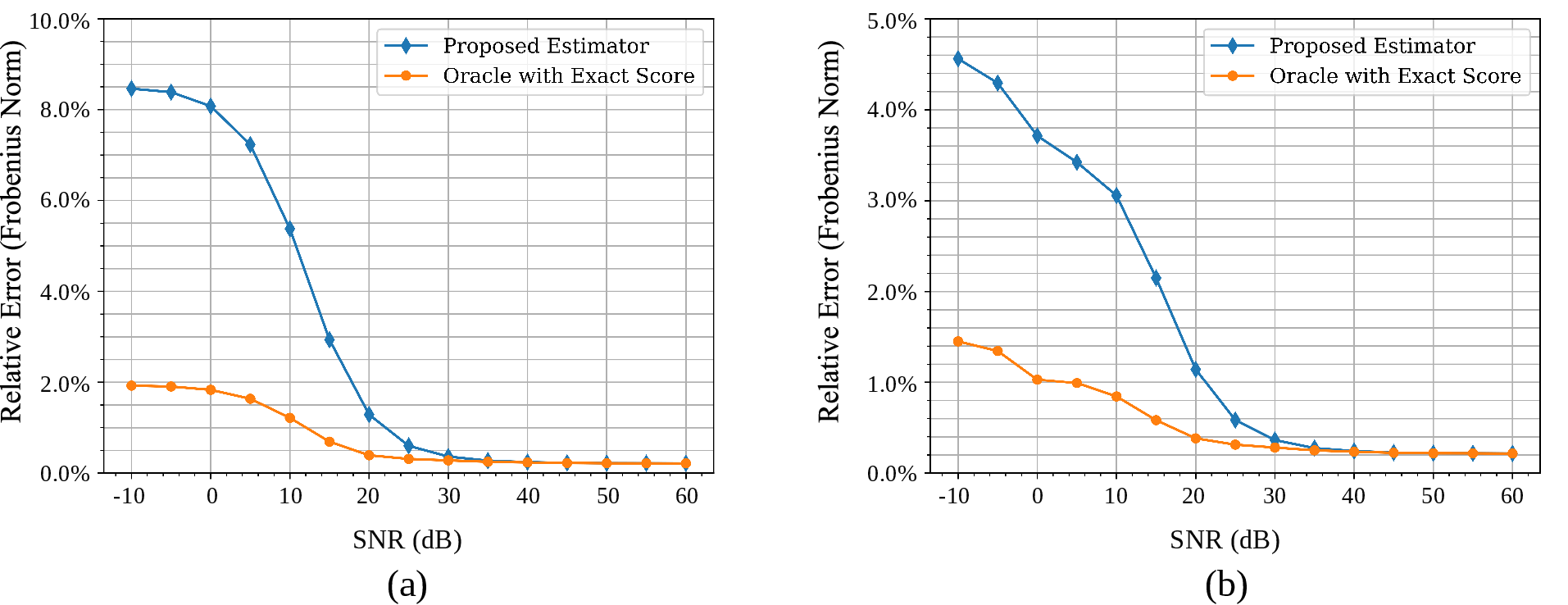}
    \caption{Relative error in the estimation of (a) the Bayesian information $\J_B$ and (b) the Bayesian CRB $\V_B$ as a function of the signal-to-noise ratio (SNR) for the phase offset estimation problem, where here the proposed approach was implemented with $N = 10^4$ samples. The performance of an oracle that has access to the ground truth score function $\nabla_{\x} \log p(\x)$ is included for reference. As can be seen, the proposed approach provides a relative error in the Bayesian CRB estimation of under $5\%$ at all SNR levels. Here note that the performance of the proposed estimator for this problem is similar to that of the denoising example, which again demonstrates the efficacy of the proposed approach. }
    \label{fig:revsnr_phase}
\end{figure}

%% file: Sections/conclusion.tex
\section{Conclusions and Future Directions}
\label{sec:conclusion}

This paper presented a new data-driven approach for the Bayesian CRB estimation. The proposed approach incorporates score matching, a statistical estimation technique that underpins a new class of state-of-the-art generative modeling techniques, to model the prior distribution. To characterize the proposed estimator, we considered two different modeling regimes: a classical parametric regime, and a neural network modeling regime, where the score model is a neural network. In both regimes, we have established non-asymptotic bounds on the errors of score matching and Bayesian CRB estimation. Our proofs draw upon results in empirical process theory, building off of both classical theory and recently developed techniques for characterizing neural networks. We evaluated the performance of the proposed estimator on two example problems: a signal denoising problem with a Gaussian mixture prior, and a dynamic phase offset estimation problem. The results demonstrate that the proposed estimator provides accurate and robust estimation performance for different prior distributions and likelihood functions.  

It is worth mentioning several interesting open problems that require further in-depth investigation. First, the proposed method assumes that the data model is given, along with a set of i.i.d. samples from the prior distribution. In future work, it would be useful to generalize the proposed approach to the estimation of the Bayesian CRB or the classical CRB in the setting where the data model is unknown but only a set of samples $\{ \x_i, \y_i \}_{i=1}^N$, $\x_i, \y_i \sim p(\x, \y)$, are available. In this setting, a key technical challenge would be estimating the score of the likelihood function, which, unlike the prior score, depends on both the unknown parameters $\x$ and the measurements $\y$. It is also of interest to investigate the extension of the proposed approach to other score-based Bayesian bounds, such as the Bayesian Bhattacharyya bound \cite{trees_2007_bayesian}.  Second, the numerical experiments in this work consider two relatively low-dimensional statistical estimation problems with known prior distributions. In future work, it would be interesting to apply the proposed approach to high-dimensional real-world problems, such as those in the context of imaging, where Cram{\'e}r-Rao type bounds are often used to guide the design of imaging systems and benchmark the performance of statistical estimators (see, e.g., \cite{roessl2009cramer, zhao_2014_model, zhao_2019_optimal, scope2022efficient,crafts2024bayesian}). Here key challenges include the lack of a ground truth prior when analyzing the error in our estimator and the high dimensionality of the problem, which may necessitate the incorporation of regularization or other techniques to yield accurate estimation of Cram{\'e}r-Rao type bounds. Finally, regarding the non-asymptotic estimator bounds, interesting future directions include the development of information-theoretic lower bounds that would provide insights into the optimality of our results, and the generalization of the neural network model based bounds in Section \ref{sec:convergeNN} to other advanced neural network architectures, such as transformer architectures \cite{attention_2017}.

%% file: Appendices/classical.tex
\section{Proofs for Section \ref{sec:classicalconverge}}
\label{apppendix:classical}

This appendix collects various proofs omitted from Section \ref{sec:classicalconverge} in the main text. 

\begin{lemma}
\label{lemma:lipschitz}
Suppose $s(\x; \boldtheta)$ is sufficiently smooth (Assumption \ref{assumption:lipschitz1}). Then $\ell(\x; \boldtheta)$ is Lipschitz continuous with Lipschitz constant $L(\x)$ and $\mathbb{E}_{\x} \left [ L(\x)^2 \right ] < \infty$. 
\end{lemma}

\begin{proof}
We have that 
\begin{align*}
| \ell(\x; \boldtheta_1) - \ell(\x; \boldtheta_2) | &= \left | \text{tr}\left(\nabla_{\x} s(\x; \boldtheta_1) \right ) + \frac{1}{2} s(\x; \boldtheta_1)^Ts(\x; \boldtheta_1)  - \left (\text{tr}\left(\nabla_{\x} s(\x; \boldtheta_2) \right ) + \frac{1}{2} s(\x; \boldtheta_2)^Ts(\x; \boldtheta_2) \right )  \right | \nonumber \\
&\leq \left | \text{tr}\left(\nabla_{\x} s(\x; \boldtheta_1) \right ) - \text{tr}\left(\nabla_{\x} s(\x; \boldtheta_2) \right ) \right| + \frac{1}{2} \left |   s(\x; \boldtheta_1)^Ts(\x; \boldtheta_1)  - s(\x; \boldtheta_2)^Ts(\x; \boldtheta_2) \right | \nonumber \\
&= \sum_{i=1}^D \left | \nabla_\x s(\x; \boldtheta_1)_{ii} - \nabla_\x s(\x; \boldtheta_2)_{ii} \right | + \frac{1}{2} \sum_{i=1}^D | s(\x; \boldtheta_1)_i^2 -  s(\x; \boldtheta_2)_i^2 | \nonumber \\
& \numIneq{(i)} \sqrt{D} \sqrt{ \sum_{i=1}^D  \left [ \nabla_\x s(\x; \boldtheta_1)_{ii} - \nabla_\x s(\x; \boldtheta_2)_{ii} \right]^2 } + \frac{\sqrt{D}}{2} \sqrt{ \sum_{i=1}^D \left [ s(\x; \boldtheta_1)_i^2 -  s(\x; \boldtheta_2)_i^2 \right ]^2} \nonumber \\
& \leq \sqrt{D} \| \nabla_\x s(\x; \boldtheta_1) - \nabla_\x s(\x; \boldtheta_2) \|_2 + \frac{\sqrt{D}}{2} \|  s(\x; \boldtheta_1) s(\x; \boldtheta_1)^T -  s(\x; \boldtheta_2) s(\x; \boldtheta_2)^T \|_2  \nonumber \\
& \numIneq{(ii)}  \left [ \sqrt{D} L_1(\x) + \frac{\sqrt{D}}{2} L_2(\x) \right ]\| \boldtheta_1 - \boldtheta_2 \|_2, 
\end{align*}
where $(i)$ is due to the equivalence of finite dimensional norms and $(ii)$ is by Assumption \ref{assumption:lipschitz1}. So $\ell(\x; \boldtheta)$ is Lipschitz with a Lipschitz constant bounded by $L(\x) \triangleq \sqrt{D} (L_1(\x) + L_2(\x)/2)$. Further, we have that 
\begin{align*}
\mathbb{E}_{\x} \left [ L(\x)^2 \right ] &= D \mathbb{E}_{\x} \left [ L_1^2(\x) + \frac{L_2(\x)^2}{4} + L_1(\x) L_2(\x) \right ] \nonumber \\
& \numIneq{(i)} D \mathbb{E}_{\x} \left [ \frac{3}{2}L_1^2(\x) + \frac{3 L_2(\x)^2}{4} \right ] \numStrictIneq{(ii)} \infty,
\end{align*}
where $(i)$ is by Young's inequality and $(ii)$ is due to Assumption \ref{assumption:lipschitz1}. 
\end{proof}

\begin{proof}[Proof of Theorem \ref{lemma:uniform_converge}]

By Jensen's inequality, we have that
\begin{align*}
    \mathbb{E}_{\x} \left [ \sup_{\boldtheta \in \boldsymbol{\Theta}} \left |\hat{J}(\boldtheta; \x_1^N) - J(\boldtheta) \right | \right ] & = \mathbb{E}_{\x} \left [\sup_{\boldtheta \in \boldsymbol{\Theta} } \left |\hat{J}(\boldtheta; \x_1^N) - \mathbb{E}_{\x} \left [\hat{J}(\boldtheta; \x_1^N) \right ] \right | \right ] \nonumber \\
    &= \mathbb{E}_{\x} \left [\sup_{\boldtheta \in \boldsymbol{\Theta} } \left |\hat{J}(\boldtheta; \x_1^N) - \mathbb{E}_{\x'} \left [\hat{J}(\boldtheta; {\x'}_1^N) \right ] \right | \right ] \nonumber \\
    &\leq \mathbb{E}_{\x, \x'} \left [\sup_{\boldtheta \in \boldsymbol{\Theta} }\left |\hat{J}(\boldtheta; \x_1^N) -   \hat{J}(\boldtheta; {\x'}_1^N)  \right | \right ],
\end{align*}
where ${\x'}_1^N$ is an independent copy of $\x_1^N$. Now let $\{ \epsilon_i \}_{i=1}^N$ be a set of independent Rademacher (symmetric Bernoulli) random variables. Since  $\ell(\x_i; \boldtheta) - \ell({\x'}_i; \boldtheta)$ is symmetric around zero, we have that 
\begin{align*}
    {\mathbb{E}_{\x, \x'}} \left [\sup_{\boldtheta \in \boldsymbol{\Theta} }\left |\hat{J}(\boldtheta; \x_1^N) -   \hat{J}(\boldtheta; {\x'}_1^N)  \right | \right ]  &= \mathbb{E}_{\x, \x'} \left [\sup_{\boldtheta \in \boldsymbol{\Theta} }  \left | \frac{1}{N} \sum_{i=1}^N \ell(\x_i; \boldtheta) - \ell({\x'}_i; \boldtheta) \right |\right ] \nonumber \\
    &= \mathbb{E}_{\x, \x', \epsilon} \left [ \sup_{\boldtheta \in \boldsymbol{\Theta}} \left |\frac{1}{N} \sum_{i=1}^N \epsilon_i \left [ \ell(\x_i; \boldtheta) - \ell({\x'}_i; \boldtheta) \right ] \right | \right ].
\end{align*}
We then fix a $\boldtheta' \in \boldsymbol{\Theta}$ to obtain
\begin{align}
\label{eq:lemma2_triangle_splitting}
\mathbb{E}_{\x, \x', \epsilon} &\left [ \sup_{\boldtheta \in \boldsymbol{\Theta}} \left |\frac{1}{N} \sum_{i=1}^N \epsilon_i \left [ \ell(\x_i; \boldtheta) - \ell({\x'}_i; \boldtheta) \right ] \right | \right ] \nonumber \\
&= \mathbb{E}_{\x, \x', \epsilon} \left [ \sup_{\boldtheta \in \boldsymbol{\Theta}} \left |\frac{1}{N} \sum_{i=1}^N \epsilon_i \left [ \ell(\x_i; \boldtheta) - \ell(\x_i; \boldtheta') + \ell(\x_i; \boldtheta') - \ell({\x'}_i; \boldtheta) \right ] \right | \right ] \nonumber \\
&\leq \mathbb{E}_{\x, \epsilon} \left [ \sup_{\boldtheta \in \boldsymbol{\Theta} } \left | \frac{1}{N} \sum_{i=1}^N \epsilon_i \left [  \ell(\x_i; \boldtheta) - \ell(\x_i; \boldtheta') \right ] \right | \right ] + \mathbb{E}_{\x, \x' \epsilon} \left [ \sup_{\boldtheta \in \boldsymbol{\Theta} } \left | \frac{1}{N} \sum_{i=1}^N \epsilon_i \left [ \ell(\x_i; \boldtheta') - \ell({\x'}_i; \boldtheta) \right ] \right | \right ] \nonumber \\
&= \mathbb{E}_{\x, \epsilon} \left [ \sup_{\boldtheta \in \boldsymbol{\Theta} } \left | \frac{1}{N} \sum_{i=1}^N \epsilon_i \left [  \ell(\x_i; \boldtheta) - \ell(\x_i; \boldtheta') \right ] \right | \right ] \nonumber \\
&\quad \quad + \mathbb{E}_{\x, \x' \epsilon} \left [ \sup_{\boldtheta \in \boldsymbol{\Theta} } \left | \frac{1}{N} \sum_{i=1}^N \epsilon_i \left [ \ell(\x_i; \boldtheta') - \ell({\x'}_i; \boldtheta') + \ell({\x'}_i; \boldtheta') - \ell({\x'}_i; \boldtheta) \right ] \right | \right ] \nonumber \\
 &\leq \mathbb{E}_{\x, \epsilon} \left [ \sup_{\boldtheta \in \boldsymbol{\Theta} } \left | \frac{1}{N} \sum_{i=1}^N \epsilon_i \left [  \ell(\x_i; \boldtheta) - \ell(\x_i; \boldtheta') \right ] \right | \right ] + \mathbb{E}_{\x, \x', \epsilon} \left [ \left | \frac{1}{N} \sum_{i=1}^N \epsilon_i \left [ \ell(\x_i; \boldtheta') - \ell({\x'}_i; \boldtheta') \right ] \right | \right ] \nonumber \\
 & \quad \quad +  \mathbb{E}_{\x', \epsilon } \left [\sup_{\boldtheta \in \boldsymbol{\Theta} } \left | \frac{1}{N} \sum_{i=1}^N \epsilon_i \left [\ell(\x'_i; \boldtheta') - \ell(\x'_i; \boldtheta) \right ] \right | \right ] \nonumber \\
    &= 2 \mathbb{E}_{\x, \epsilon} \left [ \sup_{\boldtheta \in \boldsymbol{\Theta} } \left | \frac{1}{N} \sum_{i=1}^N \epsilon_i \left [  \ell(\x_i; \boldtheta) - \ell(\x_i; \boldtheta') \right ] \right | \right ] +  \mathbb{E}_{\x, \x', \epsilon } \left [ \left | \frac{1}{N} \sum_{i=1}^N \epsilon_i \left [\ell(\x'_i; \boldtheta') - \ell(\x_i; \boldtheta') \right ] \right | \right ] 
\end{align}
through repeated application of the triangle inequality. 

The above expression has two terms. The second term does not involve a supremum and can be straightforwardly bounded by the law of large numbers. Specifically, we have that 
\begin{align}
\label{eq:lemma2_uniform_law_bound}
    \mathbb{E}_{\x, \x', \epsilon }  \left [ \left | \frac{1}{N} \sum_{i=1}^N \epsilon_i \left [\ell({\x'}_i; \boldtheta') - \ell(\x_i; \boldtheta') \right ] \right | \right ] &= \mathbb{E}_{\x, \x'} \left [ \left | \frac{1}{N} \sum_{i=1}^N \left [\ell({\x'}_i; \boldtheta') - \ell(\x_i; \boldtheta') \right ] \right | \right ] \nonumber \\
    &= \mathbb{E}_{\x, \x'} \left [ \left | \frac{1}{N} \sum_{i=1}^N \left [\ell({\x'}_i; \boldtheta') - J(\boldtheta') + J(\boldtheta') -  \ell(\x_i; \boldtheta') \right ] \right | \right ] \nonumber \\
    &\leq \mathbb{E}_{\x} \left [ \left | \hat{J}(\boldtheta', \x_1^N) - J(\boldtheta') \right | \right ] + \mathbb{E}_{\x'} \left [ \left | \hat{J}(\boldtheta', {\x'}_1^N) - J(\boldtheta') \right | \right ] \nonumber \\
    &= 2\mathbb{E}_{\x} \left [ \left | \hat{J}(\boldtheta', \x_1^N) - J(\boldtheta') \right | \right ] = O\left ( \frac{1}{\sqrt{N}} \right ).
\end{align}

To bound the first term, note that for any fixed $\x_1^N$, the random variable $\frac{1}{N} \sum_{i=1}^N \epsilon_i \ell(\x_i; \boldtheta)$ is a sub-Gaussian process (see, e.g., Definition 5.16 in \cite{wainwright_2019_high}) with respect to $\boldtheta$, i.e., for any $\boldtheta_1, \boldtheta_2 \in \boldsymbol{\Theta}$, $\x_1^N$, and $\lambda \in \mathbb{R}$, we have that 
\begin{align*}
    \mathbb{E}_{\epsilon} \left [ e^{\lambda \frac{1}{N} \sum_{i=1}^N \epsilon_i \left [\ell(\x_i; \boldtheta_1) - \ell(\x_i; \boldtheta_2) \right ]}\right ] & \numEq{(i)} \prod_{i=1}^N \mathbb{E}_{\epsilon} \left [ e^{\lambda \frac{1}{N} \epsilon_i \left [\ell(\x_i; \boldtheta_1) - \ell(\x_i; \boldtheta_2) \right ] }\right ] \nonumber \\
    &\numIneq{(ii)} \prod_{i=1}^N \mathrm{exp} \left ( \frac{\lambda^2}{2N^2} \left [ \ell(\x_i; \boldtheta_1) - \ell(\x_i; \boldtheta_2)\right ]^2 \right ) \nonumber \\
    &\numIneq{(iii)}  \prod_{i=1}^N \mathrm{exp} \left ( \frac{\lambda^2}{2N^2} L^2(\x_i) \| \boldtheta_1 - \boldtheta_2 \|_2^2 \right ) \nonumber \\
    &= \mathrm{exp} \left (\sum_{i=1}^N \frac{\lambda^2}{2N^2} L^2(\x_i) \| \boldtheta_1 - \boldtheta_2 \|_2^2  \right ). 
\end{align*}
Here $(i)$ is by independence of the $\epsilon_i$, $(ii)$ is because $\mathbb{E}_{\epsilon}\left [e^{\beta \epsilon} \right ] \leq e^{\beta^2/2}$ for all $\beta \in \mathbb{R}$, and $(iii)$ is by Lemma \ref{lemma:lipschitz}. This shows that $\frac{1}{N} \sum_{i=1}^N \epsilon_i \ell(\x_i; \boldtheta)$ is a sub-Gaussian process with metric 
$$
d(\boldtheta_1, \boldtheta_2) \triangleq \frac{1}{\sqrt{N}} \sqrt{ \frac{1}{N} \sum_{i=1}^N L^2(\x_i) } \| \boldtheta_1 - \boldtheta_2 \|_2. 
$$
Now, since $\boldsymbol{\Theta}$ is compact, the diameter of $\boldsymbol{\Theta}$ with respect to the Euclidean norm is finite, and we denote it as $\mathrm{diam}(\boldsymbol{\Theta})$. Using Dudley's entropy integral \cite{dudley_1967_sizes}, we can conclude that 
\begin{equation}
\label{eq:lemma2_dudley}
\mathbb{E}_{\x, \epsilon} \left [ \sup_{\boldtheta \in \boldsymbol{\Theta} } \left | \frac{1}{N} \sum_{i=1}^N \epsilon_i \left [  \ell(\x_i; \boldtheta) - \ell(\x_i; \boldtheta') \right ] \right | \right ] \leq O(1) \mathbb{E}_{\x} \left [\int_0^{ \frac{1}{\sqrt{N}} \sqrt{\frac{1}{N}\sum_{i=1}^N L^2(\x_i) } \mathrm{diam}(\boldsymbol{\Theta})} \sqrt{\log N(\boldsymbol{\Theta}, d, \epsilon)}  \; d \epsilon \right ].    
\end{equation}
Here $\log N(\boldsymbol{\Theta}, d, \epsilon)$ is the metric entropy, i.e., the log of the $\epsilon$-covering number, of $\boldsymbol{\Theta}$ with respect to the metric $d(\boldtheta_1, \boldtheta_2) = \frac{1}{\sqrt{N}} \sqrt{ \frac{1}{N} \sum_{i=1}^N L^2(\x_i) } \| \boldtheta_1 - \boldtheta_2 \|_2$. \\

It is known that the $\epsilon$-covering number of $\boldsymbol{\Theta}$ with the canonical Euclidean metric satisfies 
\begin{equation*}
    N(\boldsymbol{\Theta}, \| \cdot \|_2, \epsilon) \leq \left ( 1 + \frac{\mathrm{diam}(\boldsymbol{\Theta})}{\epsilon} \right )^P.
\end{equation*}
We can therefore bound $N(\boldsymbol{\Theta}, d, \epsilon)$ by 
\begin{equation*}
    N(\boldsymbol{\Theta}, d, \epsilon) \leq \left ( 1 + \sqrt{\frac{1}{N} \sum_{i=1}^N L^2(\x_i)}\frac{\mathrm{diam}(\boldsymbol{\Theta})}{\sqrt{N} \epsilon} \right )^P. 
\end{equation*}
Using this bound, we can bound the metric entropy integral as follows: 
\begin{align}
    \label{eq:lemma2_metricBound}
   &\int_0^{ \frac{1}{\sqrt{N}} \sqrt{\frac{1}{N}\sum_{i=1}^N L^2(\x_i) } \mathrm{diam}(\boldsymbol{\Theta})} \sqrt{\log N(\boldsymbol{\Theta}, d, \epsilon)}  \; d \epsilon \nonumber \\
   &\quad \quad \quad  \leq  \int_0^{ \frac{1}{\sqrt{N}} \sqrt{\frac{1}{N}\sum_{i=1}^N L^2(\x_i) } \mathrm{diam}(\boldsymbol{\Theta})} \sqrt{P \log \left (1 + \sqrt{\frac{1}{N} \sum_{i=1}^N L^2(\x_i)}\frac{\mathrm{diam}(\boldsymbol{\Theta})}{\sqrt{N} \epsilon}   \right )} d \epsilon \nonumber \\
   & \quad \quad \quad \numIneq{(i)} \int_0^{ \frac{1}{\sqrt{N}} \sqrt{\frac{1}{N}\sum_{i=1}^N L^2(\x_i) } \mathrm{diam}(\boldsymbol{\Theta})} \sqrt{\sqrt{\frac{1}{N}\sum_{i=1}^N L^2(\x_i)} \frac{P \mathrm{diam}(\boldsymbol{\Theta})}{\sqrt{N} \epsilon}} \; d \epsilon \nonumber \\
    & \quad \quad \quad = 2 \sqrt{\frac{1}{N}\sum_{i=1}^N L^2(\x_i)} \sqrt{\frac{P}{N}} \mathrm{diam}(\boldsymbol{\Theta}),
\end{align} 
where here $(i)$ holds because $\log(1 + x) \leq x$ for all $x \geq 0$. Finally, combining \eqref{eq:lemma2_triangle_splitting}, \eqref{eq:lemma2_uniform_law_bound}, \eqref{eq:lemma2_dudley} and \eqref{eq:lemma2_metricBound} gives 
\begin{align*}
    \mathbb{E}_{\x} \left [ \sup_{\boldtheta \in \boldsymbol{\Theta}} \left |\hat{J}(\boldtheta; \x_1^N) - J(\boldtheta) \right | \right ] &\leq 4 O(1) \mathbb{E}_{\x} \left [ \sqrt{\frac{1}{N}\sum_{i=1}^N L^2(\x_i)} \sqrt{\frac{P}{N}} \mathrm{diam}(\boldsymbol{\Theta}) \right ] + O\left (\frac{1}{\sqrt{N}} \right ) \nonumber \\
    & \numIneq{(i)} O(1) \sqrt{\frac{P}{N}} \mathrm{diam}(\boldsymbol{\Theta}) \sqrt{\mathbb{E}_{\x} \left [ \frac{1}{N}\sum_{i=1}^N L^2(\x_i) \right ]} + O \left (\frac{1}{\sqrt{N}} \right )  \nonumber \\
    & \numEq{(ii)} O \left ( \frac{1 + \mathrm{diam} \left ( \boldsymbol{\Theta} \right ) \sqrt{P}}{\sqrt{N}} \right ),
\end{align*}
where $(i)$ is due to Jensen's inequality and $(ii)$ holds because Lemma \ref{lemma:lipschitz} guarantees that the expectation is finite. Setting $C_S \triangleq O(1 + \mathrm{diam} \left ( \boldsymbol{\Theta} \right ) \sqrt{P})$ completes the proof. 
\end{proof}

\begin{proof}[Proof of Corollary \ref{lemma:obj_diff_bound}]
By Jensen's inequality, we have that
\begin{align*}
\label{eq:gap_initialBound} 
\mathbb{E}_{\x} \left [\sup_{\|\boldtheta - \boldtheta^*\|_2 \leq \delta} \left | \Delta_N(\boldtheta) \right |  \right ] &= \mathbb{E}_{\x}\left [\sup_{\|\boldtheta - \boldtheta^*\|_2 \leq \delta} \left | (\hat{J}(\boldtheta; \x_1^N) - J(\boldtheta)) - (\hat{J}(\boldtheta^*; \x_1^N) - J(\boldtheta^*)) \right |  \right ] \nonumber \\
&= \mathbb{E}_{\x}\left [\sup_{\|\boldtheta - \boldtheta^*\|_2 \leq \delta} \left | (\hat{J}(\boldtheta; \x_1^N) - \hat{J}(\boldtheta^*; \x_1^N)) - \mathbb{E}_{\x'} \left[ \hat{J}(\boldtheta; {\x'}_1^N) - \hat{J}(\boldtheta^*; {\x'}_1^N) \right ] \right |  \right ] \nonumber \\
&\leq \mathbb{E}_{\x, \x'} \left [\sup_{\|\boldtheta - \boldtheta^*\|_2 \leq \delta} \left | (\hat{J}(\boldtheta; \x_1^N) - \hat{J}(\boldtheta; {\x'}_1^N)) - (\hat{J}(\boldtheta^*; \x_1^N) - \hat{J}(\boldtheta^*; {\x'}_1^N)) \right| \right ]. 
\end{align*}
Now let $\{ \epsilon_i \}_{i=1}^N$ be a set of independent Rademacher random variables. We have that 
\begin{align*}
  \mathbb{E}_{\x, \x'}& \left [\sup_{\|\boldtheta - \boldtheta^*\|_2 \leq \delta} \left | (\hat{J}(\boldtheta; \x_1^N) - \hat{J}(\boldtheta; {\x'}_1^N)) - (\hat{J}(\boldtheta^*; \x_1^N) - \hat{J}(\boldtheta^*; {\x'}_1^N)) \right| \right ]  \nonumber \\
  &\hspace{2cm} = \mathbb{E}_{\x, \x'} \left [\sup_{\|\boldtheta - \boldtheta^*\|_2 \leq \delta} \left | \frac{1}{N} \sum_{i=1}^N (\ell(\x_i; \boldtheta) - \ell({\x'}_i; \boldtheta)) - (\ell(\x_i; \boldtheta^*) - \ell({\x'}_i; \boldtheta^*)) \right|  \right ] \nonumber \\
    &\hspace{2cm} \numEq{(i)} \mathbb{E}_{\x, \x', \epsilon} \left [\sup_{\|\boldtheta - \boldtheta^*\|_2 \leq \delta} \left | \frac{1}{N} \sum_{i=1}^N \epsilon_i (\ell(\x_i; \boldtheta) - \ell({\x'}_i; \boldtheta)) - \epsilon_i(\ell(\x_i; \boldtheta^*) - \ell({\x'}_i; \boldtheta^*)) \right|  \right ] \nonumber \\
    &\hspace{2cm} \leq \mathbb{E}_{\x, \epsilon} \left [\sup_{\|\boldtheta - \boldtheta^*\|_2 \leq \delta} \left | \frac{1}{N} \sum_{i=1}^N \epsilon_i (\ell(\x_i; \boldtheta) - \ell(\x_i; \boldtheta^*)) \right | \right ]  \nonumber \\
    &\hspace{3cm} + \mathbb{E}_{\x', \epsilon} \left [\sup_{\|\boldtheta - \boldtheta^*\|_2 \leq \delta} \left | \frac{1}{N} \sum_{i=1}^N \epsilon_i (\ell({\x'}_i; \boldtheta) - \ell({\x'}_i; \boldtheta^*)) \right | \right ]
 \nonumber \\
 &\hspace{2cm} =  2 \mathbb{E}_{\x, \epsilon} \left [\sup_{\|\boldtheta - \boldtheta^*\|_2 \leq \delta} \left | \frac{1}{N} \sum_{i=1}^N \epsilon_i (\ell(\x_i; \boldtheta) - \ell(\x_i; \boldtheta^*)) \right | \right ],
 \end{align*}
where $(i)$ holds because $\ell(\x_i; \boldtheta) - \ell({\x'}_i; \boldtheta)$ is symmetric around zero. 

The rest of the proof is similar to that of Theorem \ref{eq:lemma_2} and is therefore omitted. 
 
\end{proof}

\begin{proof}[Proof of Theorem \ref{theorem:finite_sample_bound}]
The proof is based on a partition of the parameter space into the following collection of ``shells'':
\begin{equation*}
S_{N, j} \triangleq \{ \boldtheta \in \boldsymbol{\Theta} \, : \, 2^j < \sqrt{N} \| \boldtheta - \boldtheta^* \|_2 \leq 2^{j+1} \}.
\end{equation*}
Specifically, letting $T \triangleq 8 C_{\boldtheta} \sqrt{P} / \epsilon \lambda$, we have that 
\begin{align}
    \label{eq:theorem2_break}
    \mathbb{P}\left[\sqrt{N} \| \hat{\boldtheta}_N - \boldtheta^* \|_2 > T \right ] & \leq \mathbb{P}\left [\sqrt{N} \| \hat{\boldtheta}_N - \boldtheta^* \|_2 >  T, \| \hat{\boldtheta}_N - \boldtheta^* \|_2 \leq \eta \right ] + \mathbb{P} \left [\| \hat{\boldtheta}_N - \boldtheta^* \|_2 > \eta \right ] \nonumber \\
    &\leq \mathbb{P}[\exists j \geq \log_2 \left ( T
    \right) \text{ such that } 2^j \leq \sqrt{N} \eta \text{ and } \hat{\boldtheta}_N \in S_{N, j} ] + \mathbb{P} \left [\| \hat{\boldtheta}_N - \boldtheta^* \|_2 > \eta \right ] \nonumber \\
    &\leq \sum_{j \geq \log_2 \left (T \right ), 2^j \leq \sqrt{N} \eta} \mathbb{P}[\hat{\boldtheta}_N \in S_{N, j}] + \mathbb{P} \left [\| \hat{\boldtheta}_N - \boldtheta^* \|_2 > \eta \right ]
\end{align}
where $\eta$ is defined by Assumption \ref{assumption:locally_quadratic}. By Lemma \ref{lemma_2}, we have that for all $N \geq 16 C_S^2/ \epsilon^2 \lambda^2 \eta^4$, it holds that 
\begin{equation}
\label{eq:consistency_statement}
\mathbb{P}\left [ \|\hat{\boldtheta}_N - \boldtheta^* \|_2 > \eta \right ] \leq \frac{\epsilon}{2},
\end{equation}
which bounds the second term in \eqref{eq:theorem2_break}.

To bound the first term, assume $\hat{\boldtheta}_N \in S_{N, j}$. Then there exists $\boldtheta \in S_{N, j}$ such that $\hat{J}(\boldtheta; \x_1^N) \leq \hat{J}(\boldtheta^*; \x_1^N)$. By Assumption \ref{assumption:locally_quadratic}, we have that 
\begin{align*}
    \Delta_N(\boldtheta) &= (\hat{J}(\boldtheta; \x_1^N) -\hat{J}(\boldtheta^*; \x_1^N)) + (J(\boldtheta^*) - J(\boldtheta)) \nonumber \\
    &\leq J(\boldtheta^*) - J(\boldtheta) \leq - \lambda \| \boldtheta - \boldtheta^* \|_2^2,
\end{align*}
which implies that 
\begin{equation*}
|\Delta_N(\boldtheta)| \geq \lambda \| \boldtheta -\boldtheta^*\|_2^2 > \lambda \frac{4^j}{N}. 
\end{equation*}
This allows us to rewrite $\mathbb{P}[\hat{\boldtheta}_N \in S_{N, j}]$ in terms of $\Delta_N(\boldtheta)$; we have that 
\begin{align*}
\mathbb{P}[\hat{\boldtheta}_N \in S_{N, j}] &\leq \mathbb{P} \left ( \exists \boldtheta \in S_{N, j} \text{ such that } |\Delta_N(\boldtheta)| > \lambda \frac{4^j}{N} \right ) \nonumber \\
&\leq  \mathbb{P}\left (\text{sup}_{\boldtheta \in S_{N, j}} |\Delta_N(\boldtheta)| > \lambda \frac{4^j}{N} \right ) \nonumber \\
&\numIneq{(i)} \frac{\mathbb{E}_{\x} \left [\text{sup}_{\boldtheta \in S_{N, j}} |\Delta_N(\boldtheta)| \right]}{\frac{\lambda 4^j}{N}}, 
\end{align*}
where $(i)$ is due to Markov's inequality. Since $\boldtheta \in S_{N, j}$ implies $\| \boldtheta - \boldtheta^*\|_2 \leq 2^{j+1} / \sqrt{N}$, by Corollary \ref{lemma:obj_diff_bound} this supremum can be bounded, i.e., we have that
\begin{equation*}
\mathbb{P}[\hat{\boldtheta}_N \in S_{N, j}]   \leq \frac{\mathbb{E}_{\x} \left [\text{sup}_{\boldtheta \in S_{N, j}} |\Delta_N(\boldtheta)| \right]}{\lambda \frac{4^j}{N}}
\leq \frac{N}{\lambda 4^j} C_{\boldtheta} \sqrt{P} \frac{ 2^{j+1}}{\sqrt{N}} \frac{1}{\sqrt{N}}
\leq \frac{C_{\boldtheta} \sqrt{P}}{\lambda 2^{j - 1}}. 
\end{equation*}
Now note that for any $\epsilon > 0$, we have that 
\begin{equation}
\sum_{j \geq \log_2 \left ( T \right ), 2^j \leq \sqrt{N} \eta} \mathbb{P}[\hat{\boldtheta}_N \in S_{N, j}] \leq \sum_{j \geq \log_2 \left ( T \right) } \frac{C_{\boldtheta} \sqrt{P}}{\lambda 2^{j - 1}} = \frac{2 C_{\boldtheta} \sqrt{P} }{\lambda } \sum_{j \geq \log_2 \left ( T \right )} \frac{1}{2^{j}} = \frac{2C_{\boldtheta} \sqrt{P}}{\lambda } \frac{2}{T} = \frac{\epsilon}{2}.
\label{eq:theorem2_lessbound}
\end{equation}
Plugging in \eqref{eq:consistency_statement} and \eqref{eq:theorem2_lessbound} into \eqref{eq:theorem2_break} gives
\begin{equation*}
\mathbb{P}\left[\sqrt{N} \| \hat{\boldtheta}_N - \boldtheta^* \|_2 > T \right ] \leq \sum_{j \geq \log_2 \left (T \right), 2^j \leq \sqrt{N} \eta} \mathbb{P}[\hat{\boldtheta}_N \in S_{N, j}] + \mathbb{P} \left [\| \hat{\boldtheta}_N - \boldtheta^* \|_2 > \eta \right ] \leq \frac{\epsilon}{2} + \frac{\epsilon}{2} = \epsilon
\end{equation*}
for all $N \geq N'$, as desired. 

\end{proof}

\begin{proof}[Proof of Theorem \ref{theorem:bayesianInfo}]
First, note that the $\J_B$ estimation error can be bounded as follows:
\begin{align*}
\|\hat{\J}_B(\x_1^N) - \J_B\|_\sigma &= \left \lVert \frac{1}{N} \sum_{i=1}^N s(\x_i; \hat{\boldtheta}_N) s(\x_i; \hat{\boldtheta}_N)^T + \hat{\J}_F(\x_i) - (\J_P + \J_D) \right \rVert_\sigma \nonumber \\
& \leq \left \lVert \frac{1}{N} \sum_{i=1}^N s(\x_i; \hat{\boldtheta}_N) s(\x_i; \hat{\boldtheta}_N )^T - \J_P \right \rVert_\sigma + \left \lVert \frac{1}{N} \sum_{i=1}^N \hat{\J}_F(\x_i) - \J_D \right \rVert_\sigma \nonumber \\
& \leq \left \lVert \frac{1}{N} \sum_{i=1}^N s(\x_i; \hat{\boldtheta}_N ) s(\x_i; \hat{\boldtheta}_N)^T - \frac{1}{N} \sum_{i=1}^N s(\x_i; \boldtheta^*) s(\x_i; \boldtheta^*)^T \right \rVert_\sigma \nonumber \\
& \quad + \left \lVert \frac{1}{N} \sum_{i=1}^N s(\x_i; \boldtheta^*) s(\x_i; \boldtheta^*)^T - \frac{1}{N} \sum_{i=1}^N \nabla_{\x} \log p(\x_i) \nabla_{\x} \log p(\x_i)^T \right \rVert_\sigma \nonumber \\
& \quad  + \left \lVert \frac{1}{N} \sum_{i=1}^N \nabla_{\x} \log p(\x_i) \nabla_{\x} \log p(\x_i)^T - \J_P \right \rVert_\sigma + \left \lVert \frac{1}{N} \sum_{i=1}^N \hat{\J}_F(\x_i) - \J_D \right \rVert_\sigma.
\end{align*} 
We now bound the four terms in the above expression. For the first term, by Theorem \ref{theorem:finite_sample_bound} we have that 
\begin{equation*}
\|\hat{\boldtheta}_N - \boldtheta^*\|_2 \leq \frac{8 C_{\boldtheta} \sqrt{P}}{\epsilon \lambda \sqrt{N}}
\end{equation*}
with probability at least $1 - \epsilon$ for all $N \geq N'$. Using this fact and Assumption \ref{assumption:lipschitz1}, it holds that
\begin{align}
\label{eq:information_sBound}
\left \lVert \frac{1}{N}\sum_{i=1}^N \left(s(\x_i; \hat{\boldtheta}_N )s(\x_i; \hat{\boldtheta}_N)^T - s(\x_i; \boldtheta^* )s(\x_i; \boldtheta^*)^T \right)\right \rVert_\sigma &\leq \left \lVert \frac{1}{N}\sum_{i=1}^N \left(s(\x_i; \hat{\boldtheta}_N )s(\x_i; \hat{\boldtheta}_N)^T - s(\x_i; \boldtheta^* )s(\x_i; \boldtheta^*)^T \right) \right \rVert_2 \nonumber \\
&\leq \frac{1}{N} \sum_{i=1}^N \left \lVert s(\x_i; \hat{\boldtheta}_N )s(\x_i; \hat{\boldtheta}_N)^T - s(\x_i; \boldtheta^* )s(\x_i; \boldtheta^*)^T \right \rVert_2 \nonumber \\
&\leq \left [ \frac{1}{N} \sum_{i=1}^N L_2(\x_i) \right] \frac{8C_{\boldtheta} \sqrt{P}}{\epsilon \lambda \sqrt{N}}
\end{align}
for all $N \geq N'$ with probability at least $1 - \epsilon$. Note that $\mu_L$ is well defined since $\mathbb{E}_{\x} \left [ L_2(\x) \right] \leq \sqrt{\mathbb{E}_{\x} \left [ L_2(\x)^2 \right]}$ by Jensen's inequality and $\mathbb{E}_{\x} \left [ L_2(\x)^2 \right]$ is finite by Assumption \ref{assumption:lipschitz1}. Further, $\sigma_L = \sqrt{\mathbb{E}_{\x} \left [ (L_2(\x) - \mu_L)^2 \right ]}$ is well defined since $\mathbb{E}_{\x} \left [ (L_2(\x) - \mu_L)^2 \right ] = \mathbb{E}_{\x} \left [L_2(\x)^2 \right] + \mu_L^2 - 2 \mu_L \mathbb{E}_{\x} \left [L_2(\x) \right] < \infty$. So by Chebyshev's inequality we have that 
\begin{equation}
\label{eq:information_Lbound}
\left | \frac{1}{N} \sum_{i=1}^N L_2(\x_i) - \mu_L \right | \leq \frac{\sigma_L}{\sqrt{\epsilon N}}.
 \end{equation}
with probability $1 - \epsilon$. Taking the union bound of \eqref{eq:information_sBound} and \eqref{eq:information_Lbound} and adjusting the confidence level, we have that for all $\epsilon > 0$ and all $N \geq N'$, it holds that 
\begin{equation}
\label{eq:score_error_bound}
 \left \lVert \frac{1}{N}\sum_{i=1}^N s(\x_i; \hat{\boldtheta}_N )s(\x_i; \hat{\boldtheta}_N)^T - s(\x_i; \boldtheta^* )s(\x_i; \boldtheta^*)^T \right \rVert_\sigma \leq \frac{16 C_{\boldtheta} \sqrt{P}}{\epsilon \lambda \sqrt{N}} \left [ \mu_L + \frac{\sqrt{2} \sigma_L}{\sqrt{\epsilon N}} \right].
\end{equation} 
with probability at least $1 - \epsilon$. 

For the second term, by Lemma \ref{lemma:expected_outer_bound} we have that 
\begin{equation*}
  \mathbb{E}_{\x} \left [ \| s(\x; \boldtheta^*)s(\x; \boldtheta^*)^T - \nabla_{\x} \log p(\x) \nabla_{\x} \log p(\x)^T \|_\sigma  \right ] \leq 2 L(\boldtheta^*) + 2 \mu_P \sqrt{2 L(\boldtheta^*) },
\end{equation*}
so by Markov's inequality with probability $1 - \epsilon$ it holds that
\begin{equation}
\label{eq:outer_product_markov}
  \left \lVert \frac{1}{N} \sum_{i=1}^N s(\x_i; \boldtheta^*) s(\x_i; \boldtheta^*)^T - \frac{1}{N} \sum_{i=1}^N \nabla_{\x} \log p(\x_i) \nabla_{\x} \log p(\x_i)^T \right \rVert_\sigma \leq \frac{1}{\epsilon} \left ( 2L(\boldtheta^*) + 2 \mu_P \sqrt{2 L(\boldtheta^*)} \right ).  
\end{equation}
For the third term, note that since $\nabla_{\x} \log p(\x)$ is sub-Gaussian with norm $C_P$ by Assumption \ref{assumption:sub_gaussian}, by Lemma \ref{lemma:sub_gaussian} we have that for all $N$, with probability at least $1 - \epsilon$,
\begin{equation}
\label{eq:jp_bound}
    \left \lVert \frac{1}{N} \sum_{i=1}^N \nabla_{\x} \log p(\x_i) \nabla_{\x} \log p(\x_i)^T - \J_P \right \rVert_\sigma  \leq C_{\boldsymbol{\Sigma}} C_P^2 \mathrm{m} \left (\sqrt{\frac{D - \log(\epsilon)}{N}} \right ).  
\end{equation}
For the fourth term, since $\nabla_{\x} \log p(\y_{ij} \mid \x_{i})$ with $\y_{ij} \sim p(\y \mid \x_i)$ is sub-Gaussian with norm $C_D$ by Assumption \ref{assumption:sub_gaussian}, we have that 
\begin{equation}
\label{eq:bayesianInfo_Jd_M}
\left \lVert \frac{1}{M} \sum_{j=1}^M \nabla_{\x} \log p(\y_{ij} \mid \x_{i}) \right \rVert_{\Psi^2} \leq \frac{1}{M} \sum_{j=1}^M \left \lVert  \nabla_{\x} \log p(\y_{ij} \mid \x_{i}) \right \rVert_{\Psi^2} = C_D, 
\end{equation}
so  $\frac{1}{M} \sum_{j=1}^M \nabla_{\x} \log p(\y_{ij} \mid \x_{i})$ is also sub-Gaussian with norm bounded by $C_D$. So if \eqref{eq:Jd_estimator2} is used to estimate $\J_D$, by Lemma \ref{lemma:sub_gaussian} the fourth term can therefore be bounded as follows with probability at least $1 - \epsilon$:
\begin{align}
\label{eq:jd_bound}
    \left \lVert \frac{1}{N} \sum_{i=1}^N \hat{\J}_F(\x_i) - \J_D \right \rVert_\sigma  &=  \left \lVert \frac{1}{N} \sum_{i=1}^N \left ( \frac{1}{M}\sum_{j=1}^M \nabla_\x \log p(\y_{ij} \mid \x_i) \right) \left (  \frac{1}{M}\sum_{j=1}^M \nabla_\x \log p(\y_{ij} \mid \x_i) \right )^T - \J_D \right \rVert_\sigma \nonumber \\
    &\leq C_{\boldsymbol{\Sigma}} C_D^2 \mathrm{m} \left (\sqrt{\frac{D - \log(\epsilon)}{N}} \right ).
\end{align}
Taking $M \to \infty$ in \eqref{eq:bayesianInfo_Jd_M} and employing the same argument shows that the above bound also holds if \eqref{eq:Jd_estimator1} is used to estimate $\J_D$. 

Finally, taking the union bound of \eqref{eq:score_error_bound}, \eqref{eq:outer_product_markov}, \eqref{eq:jp_bound}, and \eqref{eq:jd_bound}, adjusting the confidence level, and introducing the universal constant $C_1$ completes the proof.    
\end{proof}

%% file: Appendices/nn.tex
\section{Proofs for Section \ref{sec:convergeNN}}
\label{appendix:nn}

This appendix collects various proofs omitted from Section \ref{sec:convergeNN} in the main text. 

\begin{proof}[Proof of Lemma \ref{lemma:nn_model_bound}]
First note that we have that for any $\boldtheta \in \boldsymbol{\Theta}$ and $\x \in \mathcal{X}$, we have that 
\begin{equation}
\label{eq:gx_bound}
\sup_{\boldtheta \in \boldsymbol{\Theta}} \left|\ell(\x; \boldtheta)  \right | \leq \sup_{\boldtheta \in \boldsymbol{\Theta}} \frac{1}{2} \| s(\x; \boldtheta) \|_2^2 + \sup_{\boldtheta \in \boldsymbol{\Theta}} \left | \text{tr}\left(\nabla_{\x} s(\x; \boldtheta) \right ) \right |. 
\end{equation}
By Assumptions \ref{assumption:model}, \ref{assumption:bounded_support}, and \ref{assumption:bounded_weights}, the first term in the above expression satisfies
\begin{equation}
\label{eq:twoNorm_bound}
\| s(\x; \boldtheta) \|_2 \leq \left (\prod_{i=1}^L \rho_i c_i \right ) T.  
\end{equation}
For the second term, first note that the derivative of $s(\x; \boldtheta)$ can be written as
\begin{equation}
\label{eq:NNjacobian_rewritten}
\nabla_{\x} s(\x; \boldtheta) = \nabla_{\x} \sigma_L(\x) \bigg \rvert_{s_L(\x; \boldtheta)} \W_L \nabla_{\x} \sigma_{L-1}(\x; \boldtheta) \bigg \rvert_{s_{L-1}(\x; \boldtheta)} \W_{L-1} \cdots  \nabla_{\x} \sigma_1(\x) \bigg \rvert_{s_1(\x; \boldtheta)} \W_1, 
\end{equation}
where $s_l(\x; \boldtheta) \triangleq \W_l \sigma_{l-1} ( \cdots \sigma_1 (\W_1 \x))$. So we have that 
$$
\left | \text{tr}\left(\nabla_{\x} s(\x; \boldtheta) \right ) \right | \leq \left \lVert \nabla_{\x} \sigma_{1}(\x) \bigg \rvert_{s_1(\x; \boldtheta)} \W_{1} \right \rVert_2 \left \lVert  \prod_{i=1}^{L-1}  \nabla_{\x} \sigma_{l}(\x) \bigg \rvert_{s_l(\x; \boldtheta)} \W_{l} \right \rVert_{2} \leq \prod_{i=1}^L  \left \lVert \nabla_{\x} \sigma_{l}(\x) \bigg \rvert_{s_l(\x; \boldtheta)} \W_{l} \right \rVert_{2},
$$
where here we used the inequality $\left | \text{tr} \left ( \mathbf{A} \mathbf{B} \right) \right | \leq || \mathbf{A} ||_2 || \mathbf{B} ||_2$, which holds for any two matrices $\mathbf{A}$ and $\mathbf{B}$, and the submultiplicative property of the Frobenius norm. We now use the inequalities  $\| \mathbf{A}  \mathbf{B}  \|_{2} \leq \| \mathbf{A} \|_{2} \| \mathbf{B} \|_{\sigma}$ and $\| \mathbf{A} \|_{2} \leq \| \mathbf{A} \|_{2, 1}$ along with Assumptions \ref{assumption:model} and \ref{assumption:bounded_weights} to obtain 
\begin{equation}
\label{eq:bound_trace_term}
\left | \text{tr}\left(\nabla_{\x} s(\x; \boldtheta) \right ) \right | \leq \prod_{i=1}^L  \left \lVert \nabla_{\x} \sigma_{l}(\x) \bigg \rvert_{s_l(\x; \boldtheta)} \right \rVert_{\sigma} \left \lVert \W_{l} \right \rVert_{2} \leq \prod_{i=1}^L  \left \lVert \nabla_{\x} \sigma_{l}(\x) \bigg \rvert_{s_l(\x; \boldtheta)} \right \rVert_{\sigma} \left \lVert \W_{l} \right \rVert_{2, 1}  \leq \prod_{i=1}^L b_i f_i.
\end{equation}
Plugging in the bounds in \eqref{eq:twoNorm_bound} and \eqref{eq:bound_trace_term} into \eqref{eq:gx_bound} gives the desired result. 
\end{proof}

\begin{proof}[Proof of Lemma \ref{lemma:product_cover}]
     Consider the set $\mathcal{C}$ consisting of all of the elements of the form $\mathbf{C} = \mathbf{C}^L  \mathbf{C}^{L-1} \cdots \mathbf{C}^1$, where each $\mathbf{C}^l$ is an element of the cover of $\mathcal{Y}_l$. Note that the cardinality of $\mathcal{C}$ is at most $\prod_{l=1}^L v_l$, so all that needs to be shown to complete the proof is that $\mathcal{C}$ is an $\epsilon$-cover of $\mathcal{Y}$. To that end, 
     note that by the definition of $\mathcal{Y}$, any $\mathbf{Y} \in \mathcal{Y}$ can be written as $\mathbf{Y} = \mathbf{Y}^L \mathbf{Y}^{L-1} \cdots \mathbf{Y}^1$. For each $\mathbf{Y}^l$, let $\mathbf{C}^l$ be the element of the corresponding cover, and let $\mathbf{C} \in \mathcal{C}$ be given by $\mathbf{C} =  \mathbf{C}^L \mathbf{C}^{L-1} \cdots \mathbf{C}^1$. Using a telescoping sum, we have that 
     $$
     \mathbf{Y} - \mathbf{C}  = \sum_{l=1}^L \mathbf{Y}^1 \cdots \mathbf{Y}^{l-1} (\mathbf{Y}^l - \mathbf{C}^l) \mathbf{C}^{l+1} \cdots \mathbf{C}^L,
     $$
     so 
     \begin{equation}
     \label{eq:tensor_product_proof}
     \| \mathbf{Y} - \mathbf{C} \|_2 \leq \sum_{l=1}^L  \| \mathbf{Y}^1 \cdots \mathbf{Y}^{l-1} (\mathbf{Y}^l - \mathbf{C}^l) \mathbf{C}^{l+1} \cdots \mathbf{C}^L \|_2. 
     \end{equation}
     Next, observe that 
     \begin{align*}
     \| \mathbf{Y}^1 \mathbf{Y}^{l-1} (\mathbf{Y}^l - \mathbf{C}^l) \mathbf{C}^{l+1} \cdots \mathbf{C}^L \|_2 &= \left ( \sum_{i=1}^N \| \mathbf{Y}^1_i \cdots \mathbf{Y}^{l-1}_i (\mathbf{Y}^l_i - \mathbf{C}^l_i) \mathbf{C}^{l+1}_i \cdots \mathbf{C}^L_i \|_2^2 \right)^{1/2}  \\
     &\leq \left ( \sum_{i=1}^N \| \mathbf{Y}^1_i \|_2^2 \cdots \| \mathbf{Y}^{l-1}_i \|_2^2 \| \mathbf{Y}^l_i - \mathbf{C}^l_i \|_2^2 \| \mathbf{C}^{l+1}_i \|_2^2 \cdots \|\mathbf{C}^L_i \|_2^2 \right)^{1/2} \\
     &\leq \left (\sum_{i=1}^N \prod_{k \neq l} b_k^2   \| \mathbf{Y}^l_i - \mathbf{C}^l_i \|_2^2  \right)^{1/2} = \prod_{k \neq l} b_k \left ( \sum_{i=1}^N  \| \mathbf{Y}^l_i - \mathbf{C}^l_i \|_2^2  \right)^{1/2}  \\
     &= \prod_{k \neq l} b_k \|\mathbf{Y}^l - \mathbf{C}^l \|_2 \leq \epsilon_l \prod_{k \neq l} b_k.
     \end{align*}
     Plugging in the above result into Eq. \eqref{eq:tensor_product_proof} completes the proof. 
\end{proof}

\begin{proof}[Proof of Theorem \ref{theorem:bayesianInfo2}]
The proof is similar to that of Theorem \ref{theorem:bayesianInfo}. Specifically, as in Theorem \ref{theorem:bayesianInfo}, we first decompose the $\J_B$ estimation error:
\begin{align*}
\label{eq:JB_triangleNN}
\|\hat{\J}_B(\x_1^N) - \J_B\|_\sigma &= \left \lVert \frac{1}{N} \sum_{i=1}^N s(\x_i; \hat{\boldtheta}_N) s(\x_i; \hat{\boldtheta}_N)^T + \hat{\J}_F(\x_i) - (\J_P + \J_D) \right \rVert_\sigma \nonumber \\
& \leq \left \lVert \frac{1}{N} \sum_{i=1}^N s(\x_i; \hat{\boldtheta}_N) s(\x_i; \hat{\boldtheta}_N )^T - \J_P \right \rVert_\sigma + \left \lVert \frac{1}{N} \sum_{i=1}^N \hat{\J}_F(\x_i) - \J_D \right \rVert_\sigma \nonumber \\
& \leq \left \lVert \frac{1}{N} \sum_{i=1}^N s(\x_i; \hat{\boldtheta}_N) s(\x_i; \hat{\boldtheta}_N)^T - \frac{1}{N} \sum_{i=1}^N \nabla_{\x} \log p(\x_i) \nabla_{\x} \log p(\x_i)^T \right \rVert_\sigma \nonumber \\
& \quad  + \left \lVert \frac{1}{N} \sum_{i=1}^N \nabla_{\x} \log p(\x_i) \nabla_{\x} \log p(\x_i)^T - \J_P \right \rVert_\sigma + \left \lVert \frac{1}{N} \sum_{i=1}^N \hat{\J}_F(\x_i) - \J_D \right \rVert_\sigma.
\end{align*} 
We now bound the three terms in the above expression. For the first term, by Lemma \ref{corollary:expected_outer_bound} we have that 
\begin{equation*}
  \mathbb{E}_{\x} \left [ || s(\x; \hat{\boldtheta}_N)s(\x; \hat{\boldtheta}_N)^T - \nabla_{\x} \log p(\x) \nabla_{\x} \log p(\x)^T ||_\sigma  \right ] \leq 2 L(\hat{\boldtheta}_N) + 2 \mu_P \sqrt{2 L(\hat{\boldtheta}_N) },
\end{equation*}
so by Markov's inequality with probability $1 - \epsilon$ it holds that
\begin{equation}
\label{eq:outer_product_markov_NN}
  \left \lVert \frac{1}{N} \sum_{i=1}^N s(\x_i; \boldtheta^*) s(\x_i; \boldtheta^*)^T - \frac{1}{N} \sum_{i=1}^N \nabla_{\x} \log p(\x_i) \nabla_{\x} \log p(\x_i)^T \right \rVert_\sigma \leq \frac{1}{\epsilon} \left ( 2L(\hat{\boldtheta}_N) + 2 \mu_P \sqrt{2 L(\hat{\boldtheta}_N)} \right ).  
\end{equation}
Further, by Theorem \ref{theorem:mainNNscoreBound},
\begin{equation}
\label{eq:MainBoundwithCoveringNumber2}
    L(\hat{\boldtheta}_N) \leq L(\boldtheta^*) + \sqrt{\frac{8 B^2 \log(2/\epsilon)}{N}} + \frac{64 B \log(2/\epsilon)}{3N} + \frac{12 \sqrt{R}}{N} \left (1 +  \log(B N / 3 \sqrt{R} ) \right )
\end{equation}
with probability at least $1 - \epsilon$. Taking the union bound of \eqref{eq:outer_product_markov_NN} and \eqref{eq:MainBoundwithCoveringNumber2}, we have that 
\begin{align}
\label{eq:outer_product_markov_NN2}
      &\left \lVert \frac{1}{N} \sum_{i=1}^N s(\x_i; \boldtheta^*) s(\x_i; \boldtheta^*)^T - \frac{1}{N} \sum_{i=1}^N \nabla_{\x} \log p(\x_i) \nabla_{\x} \log p(\x_i)^T \right \rVert_\sigma \leq \nonumber \\
      & \hspace{.8cm} \frac{2}{\epsilon} \left ( L(\boldtheta^*) + \sqrt{\frac{8 B^2 \log(2/\epsilon)}{N}} + \frac{64 B \log(2/\epsilon)}{3N} + \frac{12 \sqrt{R}}{N} \left (1 +  \log(2 B N / 3 \sqrt{R} ) \right )\right ) \nonumber \\
      & \hspace{.8cm} + \frac{2 \sqrt{2} \mu_P}{\epsilon} \left ( \sqrt{L(\boldtheta^*)} + \left ( \frac{8 B^2 \log(2/\epsilon)}{N} \right)^{1/4} + \sqrt{\frac{64 B \log(2/\epsilon)}{3N}} + \sqrt{\frac{12 \sqrt{R}}{N} \left (1 +  \log(2 B N / 3 \sqrt{R} ) \right )} \right) 
\end{align}
with probability at least $1 - 2 \epsilon$.

To bound the second and third terms, we use the same arguments as used in proof of Theorem \ref{theorem:BayesianCRB}. Specifically, for the second term, note that since $\nabla_{\x} \log p(\x)$ is sub-Gaussian with norm $C_P$, by Lemma \ref{lemma:sub_gaussian} we have that for all $N$, with probability at least $1 - \epsilon$,
\begin{equation}
\label{eq:jp_bound_NN}
    \left \lVert \frac{1}{N} \sum_{i=1}^N \nabla_{\x} \log p(\x_i) \nabla_{\x} \log p(\x_i)^T - \J_P \right \rVert_\sigma  \leq C_{\boldsymbol{\Sigma}} C_P^2 \mathrm{m} \left (\sqrt{\frac{D - \log(\epsilon)}{N}} \right ).  
\end{equation}
For the third term, since $\nabla_{\x} \log p(\y_{ij} \mid \x_{i})$ with $\y_{ij} \sim p(\y \mid \x_i)$ is sub-Gaussian with norm $C_D$ by Assumption \ref{assumption:sub_gaussian}, we have that 
\begin{equation}
\label{eq:bayesianInfo_Jd_M2}
\left \lVert \frac{1}{M} \sum_{j=1}^M \nabla_{\x} \log p(\y_{ij} \mid \x_{i}) \right \rVert_{\Psi^2} \leq \frac{1}{M} \sum_{j=1}^M \left \lVert  \nabla_{\x} \log p(\y_{ij} \mid \x_{i}) \right \rVert_{\Psi^2} = C_D, 
\end{equation}
so  $\frac{1}{M} \sum_{j=1}^M \nabla_{\x} \log p(\y_{ij} \mid \x_{i})$ is also sub-Gaussian with norm bounded by $C_D$. So if \eqref{eq:Jd_estimator2} is used to estimate $\J_D$, by Lemma \ref{lemma:sub_gaussian} the fourth term can therefore be bounded as follows with probability at least $1 - \epsilon$:
\begin{align}
\label{eq:jd_bound_NN}
    \left \lVert \frac{1}{N} \sum_{i=1}^N \hat{\J}_F(\x_i) - \J_D \right \rVert_\sigma  &=  \left \lVert \frac{1}{N} \sum_{i=1}^N \left ( \frac{1}{M}\sum_{j=1}^M \nabla_\x \log p(\y_{ij} \mid \x_i) \right) \left (  \frac{1}{M}\sum_{j=1}^M \nabla_\x \log p(\y_{ij} \mid \x_i) \right )^T - \J_D \right \rVert_\sigma \nonumber \\
    &\leq C_{\boldsymbol{\Sigma}} C_D^2 \mathrm{m} \left (\sqrt{\frac{D - \log(\epsilon)}{N}} \right ).
\end{align}
Taking $M \to \infty$ in \eqref{eq:bayesianInfo_Jd_M} and employing the same argument shows that the above bound also holds if \eqref{eq:Jd_estimator1} is used to estimate $\J_D$. 

Finally, taking the universal bound of \eqref{eq:outer_product_markov_NN2}, \eqref{eq:jp_bound_NN}, and \eqref{eq:jd_bound_NN}, adjusting the confidence level, and introducing the universal constant completes the proof. 
\end{proof}